\documentclass[letterpaper]{article} 
\usepackage{aaai25}  
\usepackage{times}  
\usepackage{helvet}  
\usepackage{courier}  
\usepackage[hyphens]{url}  
\usepackage{graphicx} 
\usepackage{natbib}  
\usepackage{caption} 
\usepackage{algorithm}
\usepackage{amsmath,amsfonts}
\usepackage{algorithm}
\usepackage{algpseudocode}
\usepackage{lineno}
\usepackage{bm}
\usepackage{float}
\usepackage{multirow}
\usepackage{makecell}
\usepackage{tabularx}
\usepackage{amssymb}
\usepackage{stmaryrd}
\usepackage{lipsum}
\usepackage{subfigure}
\usepackage{multicol}
\usepackage{bbold}
\usepackage{amsthm}
\usepackage{newfloat}
\usepackage{listings}
\usepackage{tikz}
\usepackage{graphicx}
\usepackage{mathtools}
\usepackage{textcomp}
\usepackage{float}
\usepackage[skins]{tcolorbox}
\tcbuselibrary{breakable}
\usepackage{lipsum} 
\usepackage{xcolor}
\newtheorem{theorem}{Theorem}

\newtheorem{lemma}[theorem]{Lemma}
\newtheorem{corollary}[theorem]{Corollary}
\newtheorem{definition}[theorem]{Definition}
\newtheorem{assumption}[theorem]{Assumption}

\usepackage{newfloat}
\usepackage{listings}
\DeclareCaptionStyle{ruled}{labelfont=normalfont,labelsep=colon,strut=off} 
\floatstyle{ruled}
\newfloat{listing}{tb}{lst}{}
\floatname{listing}{Listing}
\title{Graph Partial Label Learning with Potential Cause Discovering}
\author{
    Hang Gao\textsuperscript{\rm 1}\thanks{Equal contribution.}
    Jiaguo Yuan$^*$\textsuperscript{\rm 1},
    Peng Qiao\textsuperscript{\rm 1},
    Fengge Wu\textsuperscript{\rm 1}\thanks{Corresponding author, fengge@iscas.ac.cn.},
    Changwen Zheng\textsuperscript{\rm 1},
    Huaping Liu\textsuperscript{\rm 2}
}
\affiliations{
    \textsuperscript{\rm 1}Institute of Software Chinese Academy of Science\\
    \textsuperscript{\rm 2}Department of Computer Science and Technology, Tsinghua University \\
}

\usepackage{bibentry}

\begin{document}

\maketitle

\begin{abstract}
Graph Neural Networks (GNNs) have garnered widespread attention for their potential to address the challenges posed by graph representation learning, which face complex graph-structured data across various domains. However, due to the inherent complexity and interconnectedness of graphs, accurately annotating graph data for training GNNs is extremely challenging. To address this issue, we have introduced Partial Label Learning (PLL) into graph representation learning. PLL is a critical weakly supervised learning problem where each training instance is associated with a set of candidate labels, including the ground-truth label and the additional interfering labels. PLL allows annotators to make errors, which reduces the difficulty of data labeling. Subsequently, we propose a novel graph representation learning method that enables GNN models to effectively learn discriminative information within the context of PLL. Our approach utilizes potential cause extraction to obtain graph data that holds causal relationships with the labels. By conducting auxiliary training based on the extracted graph data, our model can effectively eliminate the interfering information in the PLL scenario. We support the rationale behind our method with a series of theoretical analyses. Moreover, we conduct extensive evaluations and ablation studies on multiple datasets, demonstrating the superiority of our proposed method. \textit{Codes and appendix are available in the supplementary materials.}
\end{abstract}

\section{Introduction}

Learning graph representations has emerged as a pivotal research area within the field of computer science, garnering significant attention from both academia and industry \cite{DBLP:journals/tai/JiaoCLYYLLH23}. This burgeoning interest is primarily motivated by the pressing need to address the challenges posed by complex graph-structured data across diverse domains \cite{DBLP:journals/nn/ShenJLWXSC23,DBLP:journals/bib/ZhangWRD22,DBLP:conf/aaai/WangSLCJ0W22,DBLP:journals/nn/ShenJLWXSC23}. Currently, the emergence of GNNs, combined with data-driven approaches, provides a reliable and efficient solution for graph representation learning \cite{DBLP:conf/iclr/KipfW17}. However, accurately annotating the data required for training GNNs has long been a challenging task, owing to the inherent complexity of the data \cite{DBLP:conf/asl/PikhurkoV09}. Unlike traditional tabular data, graphs encompass intricate relationships and interconnectedness between entities, rendering significant annotation resources and error-prone \cite{DBLP:conf/nips/ShamsiVKGA22}. In order to address such an issue, unsupervised methods \cite{DBLP:journals/tai/00010YAWP021,DBLP:conf/nips/YouCSCWS20,DBLP:conf/aaai/LiuZZLP23,DBLP:conf/aaai/JiLHWZX24,DBLP:conf/aaai/LiJGQZ024} have been introduced to reduce reliance upon annotated graph data, and methods aim at learning graph representations under noisy labels \cite{DBLP:journals/tbd/LiangYCLX24, DBLP:conf/kdd/Dai0W21,DBLP:conf/icassp/YuanLQZJZ23,DBLP:journals/asc/LeeKK23} have made models more robust. Though the aforementioned methods have addressed the issue of what to do when faced with suboptimal labeling of graph data, they have not addressed whether there are ways to better label graph data under limited resources and manpower. 

\begin{figure*}[t]
	\centering
	\begin{minipage}[t]{1\linewidth}
        \centering
        \includegraphics[width=1\linewidth]{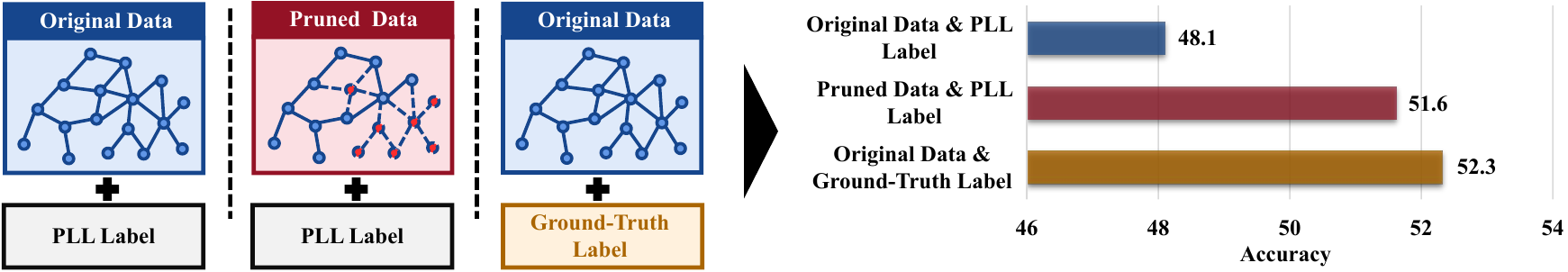}
	\end{minipage}
	\centering
    \vskip -0.1in
	\caption{Compared results. For GNN baseline, we use ARMA \cite{arma}. We adopt Graph Twitter dataset \cite{dataset:sst} to conduct the experiments.}
	\label{fig:motivation}
     \vskip -0.2in
\end{figure*}
Therefore, we introduce the PLL \cite{DBLP:phd/ndltd/Szummer02} paradigm to establish a novel and more flexible approach of graph representation learning. In PLL, each training sample is associated with multiple labels, while only one of them is the ground-truth label. Therefore, PLL allows for less-than-perfect accuracy during actual data annotation \cite{DBLP:conf/icml/GongYB22}. In other words, the PLL provides a solution that allows annotators not to ensure the complete accuracy of graph data, as long as one of the labels is correct, thereby reducing the requirements for annotation. Despite the advantages of PLL, the inherent structural complexity and semantically rich nature of graph data \cite{DBLP:conf/asl/PikhurkoV09} make the PLL problem in graph representation learning rather challenging. Highly informative inputs, combined with multiple uncertain labels, pose the risk of the model learning irrelevant information. Therefore, it would be beneficial if the task-relevant information could be identified within the graph data. Recently, methods incorporating causal learning mechanisms demonstrate that they can identify graph data causally related to the ground-truth labels \cite{DBLP:conf/iclr/WuWZ0C22, ciga, 123123}, thereby avoiding the learning of task-irrelevant information. However, these approaches are validated under general supervised scenarios. Yet, we wonder an intriguing question: \textbf{what impact does training models with graph data causally related to the ground-truth labels have in PLL scenarios?}




To investigate this issue further, we design and conduct a motivational experiment. Using the Granger Causality Criterion \cite{DBLP:journals/jmis/Dutta01} and a GNN model pre-trained with ground-truth labels, we prune the graph data to isolate a subset that consists only of data causally related to the ground-truth labels. Subsequently, we employ three different strategies to train a GNN model: 1) Training on the PLL dataset; 2) Training with pruned graphs and the labels from the PLL dataset; 3) Training under a conventional supervised learning scenario. All three models are tested on the same data set, and their results are compared. As Figure \ref{fig:motivation} reveals, the model trained with the pruned graphs significantly outperforms those trained with the original PLL dataset, and nearly matches the performance of models trained under conventional supervised scenarios with only ground-truth labels. This motivational experiment confirms a critical finding: \textbf{even in PLL scenarios where training involves inaccurate labels, effective identification and utilization of a graph causal subset can yield results comparable to conventional supervised graph representation learning} (a formal definition of graph causal subset is provided in Definition \ref{def:graph_causal_subset}). We also substantiate this finding with Theorem \ref{thm:leq}. However, the motivational experiment relies on a pre-trained GNN model with ground-truth labels and thus cannot be directly used in our task. For graph representation learning under the PLL scenario, the complexity of the data combined with label uncertainty can make locating the graph causal subset quite challenging.

To address such an issue, we propose a novel approach named \textit{\textbf{G}raph \textbf{P}otential \textbf{C}auses \textbf{D}iscovering} (GPCD). To extract the desired data, we introduce the concept of potential causes from causal theory \cite{pearl2000models}. Through theoretical validation, we demonstrate that identifying such causes can help pinpoint the graph causal subsets. GPCD analyzes the probabilistic relationships learned by the GNN on the PLL graph dataset and uses the Local Criteria For Inferring Causal Relations \cite{pearl2000models} to identify potential causes. Based on these identified causes, GPCD constructs an estimation of the graph causal subset, which then guides the training process. Moreover, the design of GPCD is substantiated through rigorous theoretical analysis. The effectiveness of GPCD is empirically validated across seven distinct datasets.

Our contributions are summarized as follows:

\begin{itemize}
\item We have discovered and demonstrated that identifying the graph causal subset can effectively ensure the efficacy of graph representation learning in the PLL context.

\item We propose an innovative method, GPCD. GPCD effectively estimates the graph causal subset with potential causes discovery, offering a novel solution for effectively utilizing graph training data under PLL scenarios.

\item We have conducted a detailed theoretical analysis and experimental validation of GPCD.

\end{itemize}




\section{Related Works}

\paragraph{Partial Label Learning.}
Methods addressing the PLL problem mainly consist of two types: average-based methods and identification-based methods. The average-based method treats each label in the candidate label set equally, assigning identical weight to them \cite{avg1,avg2,avg3}. The identification-based method treats the ground-truth label as a latent variable and identifies the ground-truth label iteratively during the learning process \cite{identify1,identify2,identify3,identify4,identify5}. Recently, confidence-based methods have achieved promising results. e.g., Pico \cite{pico} employs contrastive learning \cite{moco} and class prototype-based label disambiguation to address PLL. As a contemporary work, DEER\cite{10543125}  also attempts to address the PLL problem in graph data. However, like other PLL methods, DEER approaches the problem from the perspective of probability distributions. In contrast, this paper explores the issue from the perspective of causal discovery, focusing primarily on solving the problem by incorporating causal reasoning. 

\paragraph{Causal Representation Learning on Graphs.}
Currently, causal learning \cite{DBLP:journals/csur/GuoCLH020} finds its application within graph learning paradigms, aiming to bolster interpretability and enhance models by discerning causal relationships existing between data and labels. This array of methodologies can be delineated into two fundamental strata: inherent interpretability and invariant learning. Approaches underpinned by inherent interpretability integrate rationalization modules, typified by mechanisms like attention \cite{gan,DBLP:conf/nips/VaswaniSPUJGKP17} and pooling \cite{DBLP:conf/aaai/NguyenG18,pool,DBLP:conf/icml/WuC0L22}. In contrast, methods predicated upon invariant learning are crafted to discern pivotal subgraphs that wield a determinative role in the prognostications of GNNs \cite{gnnexplainer,DBLP:conf/icml/ChangZYJ20,DBLP:conf/icml/BevilacquaZ021}. Additionally, the paradigm of invariant learning finds utility in mitigating the challenges of generalization pertaining to out-of-distribution (OOD) data \cite{DBLP:conf/iclr/WuWZ0C22,123123,ciga}. We utilize causal learning techniques to effectively address and mitigate the impact of misleading labels encountered in graph representation learning under PLL scenarios and address the issue of how to locate graph causal subsets within data under such scenarios.






\section{Problem Formulation}


PLL investigates the challenge of learning from the training samples and their corresponding labels, which comprise both ground-truth labels and noisy labels that are indistinguishable directly. In the specific context of applying PLL to graph representation learning, we endeavor to train a dedicated GNN $f(\cdot)$, employing graph $\mathcal{G}$ and corresponding label $Y$, where $Y$ comprises both ground-truth label $Y^{*}$ and irrelevant noisy label $\widetilde{Y}$. Our problem lies in how to enable $f(\cdot)$ to model the relationship between $G$ and $Y^{*}$, while $Y^{*}$ is mixed with $\widetilde{Y}$. 

To enhance clarity throughout the paper, we denote the graph variable as $G$ and use $G_{i}$ to represent the specific $i$-th graph sample within the data set $\{G_{i}\}_{i=1}^{N}$, where $N$ is the number of graph samples. Regarding the label $Y$, it is defined as $Y \in \mathbb{R}^{K \times D}$, where $K$ indicates the number of different labels attached to each graph sample and $D$ represents the number of classes. $Y_{i}$ denotes the labels associated with $G_{i}$. Consequently, $Y_{i}$ comprises $K$ one-hot vectors, each representing a label attached to a graph sample, with only one vector corresponding to the ground-truth label. We then define $Y_{i}^{[k]}$ as the $k$-th one-hot vector label within $Y_{i}$. $Y_{i}^{[k]}$ is $D$-dimensional, corresponding to $D$ classes.

The issue we need to address is how to accurately model the relationship between $G$ and $Y^{*}$ within the context of PLL. As demonstrated in the motivating experiments, utilizing the graph causal subset $G^{*}$, which comprises only data causally related to the ground-truth labels, effectively achieves this objective; however, identifying $G^{*}$ in practical tasks is extremely challenging and elusive. Therefore, we introduce the concept of potential causes from the Local Criteria For Inferring Causal Relations \cite{pearl2000models}, and based on this, we approximate the estimation of $G^{*}$. The definition of potential causes and the criteria for locating them are as follows.
\begin{definition}  (Potential Cause) \cite{pearl2000models}
\label{def:poc}
A variable $a$ has a potential causal influence on another variable $b$ if the following conditions hold.

1. $a$ and $b$ are dependent in every context.

2. There exists a variable $c$ and a context $S$ such that (i) $a$ and $c$ are independent given $S$ and (ii) $c$ and $b$ are dependent given $S$.
\end{definition} 
Within the definition, the term ``context'' means a set of variables tied to specific values. We employ this definition to aid in the identification of variables that contribute positively to learning.

\section{Methodology}







In this section, we elaborate on the proposed method, GPCD. We first provide an overview of the approach, followed by detailed subsections. Our subsequent discussion will revolve around the process of locating and estimating the graph causal subset. Therefore, we present its formal definition.
\begin{definition}
(Graph Causal Subset) For graph sample $G$ and the corresponding ground-truth label $Y^{*}$, we denote $G^{*}$ as the \textit{graph causal subset} of $G$ if the following conditions hold.\\
1. $G^{*} \subseteq G.$ \\
2. For any $ C \subseteq G \setminus G^{*}$ (where $\setminus$ denotes the setminus operation), there exists no causal relationship between $C$ and $Y^{*}$.\\
3. For any $ X \subseteq G^{*}, X \neq \varnothing $, X holds causal relationships with $Y^{*}$.
\label{def:graph_causal_subset}
\end{definition}
As demonstrated in the conducted motivating experiments, identifying the graph causal subset $G^{*}$ within $G$ can effectively assist us in training a GNN model $f(\cdot)$ even in the presence of noise-label interference. However, locating $G^{*}$ is exceedingly challenging due to the intricate interrelations within graph data and the difficulty in distinguishing labels within PLL. Therefore, we introduce the concept of potential causes to aid in identifying $G^{*}$; we prove that locating the potential causes of $Y$ will yield an effective estimation of $G^{*}$ with Theorem \ref{thm:lag}. We have provided relevant explanations of potential causes in Definition \ref{def:poc}.

Based on the idea of using potential causes to locate the graph causal subset and then guiding accurate training of GNNs, we have designed GPCD, which can be divided into three phases: \textbf{1) pre-training}, where $f(\cdot)$ is trained to capture the probabilistic relationships within the PLL graph dataset; \textbf{2) graph causal subset estimation}, where potential causes that likely hold causal influence upon labels are identified based on $f(\cdot)$, such potential causes yield an estimation of graph causal subset; \textbf{3) auxiliary training}, where the identified potential causes are employed to guide and refine the training process. 

Figure \ref{fig:framework} demonstrates the framework of our proposed method, which will be elaborated phase by phase in the following content.

\begin{figure*}[t]
\centering
\includegraphics[width=1\textwidth]{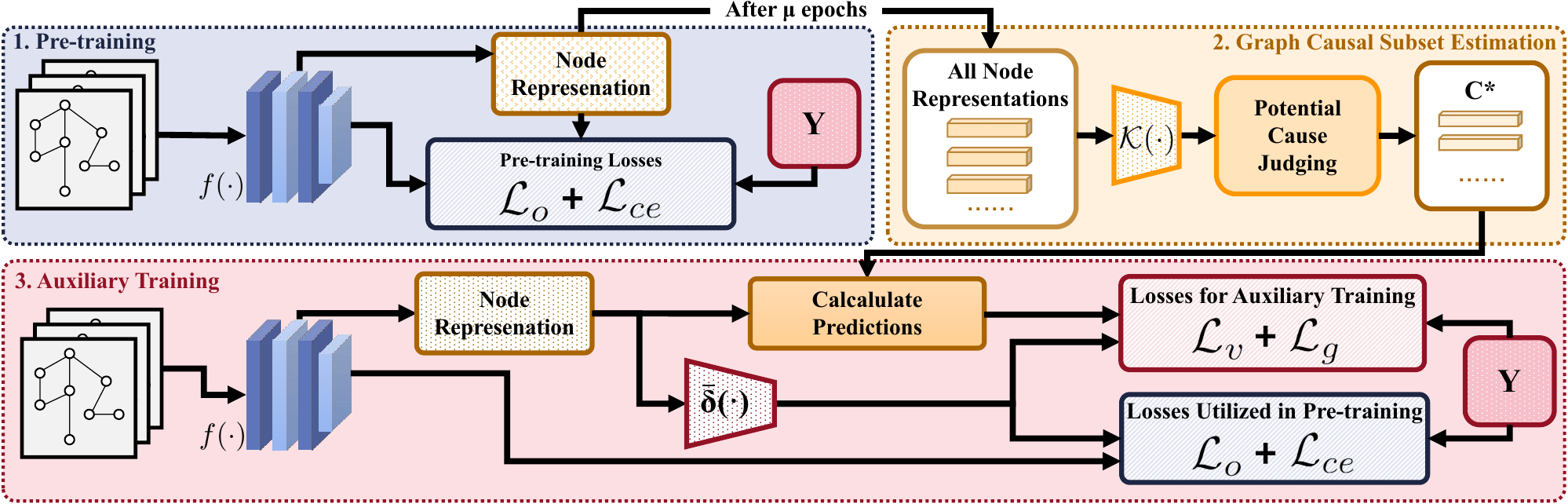} 
\vskip -0.1in
\caption{The framework of our proposed method.}

\label{fig:framework}
\vskip -0.2in
\end{figure*}


\subsection{Pre-training}
In the pre-training phase, $f(\cdot)$ is updated for $\mu$ epochs based on the cross-entropy loss $\mathcal{L}_{ce}$ and the node-level loss $\mathcal{L}_{o}$. Given the graph set $\{G_{i}\}_{i=1}^{N}$ and the corresponding labels $\{Y_{i}\}_{i=1}^{N}$, $\mathcal{L}_{ce}$ can be formulated as:
\begin{gather}
    \mathcal{L}_{ce} =  -\frac{1}{NK}\sum_{i=1}^{N} \sum_{k=1}^{K} \mathcal{H}\Big( f(G_{i}), Y_{i}^{[k]}\Big),
    \label{eq:lce}
\end{gather}
where $\mathcal{H}(\cdot)$ denotes the function that calculates the cross-entropy loss for each sample. As $Y_{i}^{[k]}$ denote the $k$-th one-hot vector label of $Y_{i}$, $\mathcal{L}_{ce}$ can be regarded as calculating cross-entropy loss over all the given labels.

As for $\mathcal{L}_{o}$, it is designed to enable $f(\cdot)$ to capture the direct relationship between node representations and graph labels, so as to locate potential causes in a fine-grained way. $\mathcal{L}_{o}$ can be formulated as:
\begin{gather}
\mathcal{L}_{o} = -\frac{1}{NK} \sum_{i=1}^{N} \sum_{k=1}
^{K} \mathcal{H}\Big( 
\nonumber\\
\text{norm}\Big(\sum_{l=1}^{|\mathcal{V}_i|} \Big(M_{i} \circ \bar{\delta}\big(\psi(G_{i})\big)\Big)^{[l]}\Big), Y_{i}^{[k]}\Big),
\label{eq:lo}
\end{gather}
where $\mathcal{V}_{i}$ denotes the node set of $G_{i}$, $|\mathcal{V}_{i}|$ denotes the number of nodes within $G_{i}$. $\psi(\cdot)$ denotes the sub-network of $f(\cdot)$ that computing node features. Additionally, $\bar{\delta}(\cdot)$, an MLP network, projects each of these node features into $D$-dimensional vectors, where $D$ is the number of classes and each vector element is constrained to the range $[-1,1]$. Consequently, the output $\bar{\delta}\big(\psi(G_{i})\big)$ corresponds to a matrix of dimensions $|\mathcal{V}_{i}| \times D$. These projection vectors are then summed up, and go through normalization $norm(\cdot)$ to calculate the probability estimation. $\circ$ denotes the element-wise product. $M_{i}$ is a mask matrix with the same size as $\bar{\delta}\big(\psi(G_{i})\big)$, which can be represented as follows:
\begin{gather}
    M_{i}^{[l,d]} = \frac{\text{abl}\big(\bar{\delta}\big(\psi(G_{i})\big)^{[l,d]}\big)}{\sum_{l=1}^{|\mathcal{V}_{i}|} \text{abl}\big(\bar{\delta}\big(\psi(G_{i})\big)^{[l,d]}\big)},
\end{gather}
where $\text{abl}(\cdot)$ denotes the calculation of absolute value, $M_{i}^{[l,d]}$ denote the element of matrix $M_{i}$ at row $l$ and column $d$. $\bar{\delta}\big(\psi(G_{i})\big)^{[l,d]}$ refers to the element at row $l$ and column $d$ of the output matrix obtained from $\bar{\delta}\big(\psi(G_{i})\big)$. 
The $M_{i}$ functions as a weighting mechanism, and helps calculate the mean value based on these weights. It guarantees that if $\bar{\delta}\big(\psi(G_{i})\big)^{[l,d]}$ equals $0$, it exerts no influence on the ultimate prediction, e.g., there can be infinite nodes representations add to $\psi(G_{i})$, as long as the corresponding output of $\bar{\delta}(\cdot)$ equals $0$, the prediction won't change. Conversely, as $\bar{\delta}\big(\psi(G_{i})\big)^{[l,d]}$ approaches extreme values, including $1$ and $-1$, its impact on the prediction significantly increases. We have demonstrated the rationality of the design of $\mathcal{L}_{o}$ through the proof provided in Theorem \ref{thm:is_potential_cause}. We then sum up $\mathcal{L}_{o}$ and $\mathcal{L}_{ce}$ as the total loss for pre-training.


\subsection{Graph Causal Subset Estimation}
After the pre-training phase, we proceed to estimate the graph causal subset $G^{*}$ of $G$. As discussed before, directly locating $G^{*}$ is infeasible; therefore, we turn to utilize the potential causes of $Y$ to estimate $G^{*}$. Benefit from the pre-training based on $\mathcal{L}_{o}$, $\bar{\delta}(\cdot)$ can now model the relationship between the node representations outputted by $\psi(\cdot)$ and label $Y$. Therefore, we could analyze the node representations to identify the potential causes of $Y$, such potential causes will yields an effective estimation of $G^{*}$ as demonstrated within Theorem \ref{thm:is_potential_cause}. Due to the large number of node features corresponding to all graph samples, it is impractical to analyze them individually. Thus, we first summarize the features through clustering, the process can be formulated as follows:
\begin{gather}
    C = \mathcal{K}\Big(\{\psi(G_{i})\}_{i=1}^{N}\Big),
\end{gather}
where $\mathcal{K}(\cdot)$ denotes the data clustering operation with K-Means \cite{macqueen1967some} and Elbow \cite{ketchen1996application} algorithms, the detailed implementation can be found in \textbf{Appendix} B.4. $C$ denotes the set of centers of grouped representations, each $\bm{c}_{j} \in C$ denotes a prototype vector that represents the center of the $j$-th group. Therefore, each $\bm{c}_{j}$ summarizes a category set $Z_{j}$ of node features, $Z_{j}$ denotes the set of node features that belongs to the $j$-th group.

Subsequently, we locate the potential causes of $Y$. Specifically, If the following condition holds:
\begin{gather}
\bar{\delta}(z)^{[d]}>\Delta, \forall z \in Z_{j},
\label{eq:ieq1}
\end{gather}
or
\begin{gather}
\bar{\delta}(z)^{[d]}<-\Delta, \forall z \in Z_{j},
\label{eq:ieq2}
\end{gather}
then we will regard the graph information corresponding to the node features within $Z_{j}$ as a potential cause of $Y$, $\Delta$ is a hyperparameter, and $\Delta \in (0,1)$. The fundamental intuition here is to identify features that exert a consistent or similar influence on model predictions, i.e., those are always greater than $\Delta$ or less than $-\Delta$. We further provide Theorem \ref{thm:is_potential_cause}, and the theoretical analysis follows as the demonstration of such judging criteria. We locate all $\bm{c}_{j}$ corresponding to $Z_{j}$ that satisfy the given criteria and select them to form a new set $C^{*}$. In the following learning, we could locate node representations that are similar to those within $C^{*}$, and then the graph information corresponding to such representations can be regarded as potential causes, thus serving as an estimation of $G^{*}$.

\subsection{Auxiliary Training} 

We devised an auxiliary training procedure to integrate the $C^{*}$ and guide the model training process. As demonstrated by the motivating experiment, $G^{*}$ and $Y$ with noisy label $\widetilde{Y}$ can also effectively train a GNN free from the influence of $\widetilde{Y}$. We further provide Theorem \ref{thm:leq} as theoretical proof to support such a claim. By utilizing $C^{*}$, we can locate the node representations corresponding to the estimation of $G^{*}$. Next, we opt to find the relationship between such representations and $Y$, then utilize the acquired relationship as guidance to conduct auxiliary training of the GNN model.

Specifically, for each node representation that similar to $\bm{c}^{*}_{j} \in C^{*}$, we directly adopt a vector $\bm{o}_{j}$ to serve as the corresponding output of $\bar{\delta}(\cdot)$, $\bm{o}_{j} \in \mathbb{R}^{1 \times D}$. Set $O = \{\bm{o}_{j}\}_{j=1}^{|C^{*}|}$ models the relationship between $Y$ and representations of the estimation of $G^{*}$. $\bm{o}_{j}$ is initialized with the value of $\bar{\delta}(\bm{c}^{*}_{j})$, and set as a learnable vector. $\bm{o}_{j}$ is updated using the following loss function:
\begin{gather}
\mathcal{L}_{v} = \sum_{i=1}^{N} \sum_{k=1}^{K} \mathcal{H}\Big( \text{epow}(\frac{1}{|\mathcal{V}_{i}|}\sum_{l=1}^{|\mathcal{V}_{i}|}\bm{w}_{i,l}, \lambda), Y_{i}^{[k]}\Big),
\label{eq:lv}
\end{gather}
where $\text{epow}([ a_1,a_2,...,a_n ]^{\top}, \lambda) = [ e^{(a_1)^{\lambda}}, e^{(a_2)^{\lambda}},...,e^{(a_n)^{\lambda}}]^{\top}$, $\bm{w}_{i,l}$ can be calculated with the following equation:
\begin{gather}
\bm w_{i,l} = 
\begin{cases}
\bm{o}_{j}, & \text{if $\text{sim}(\psi(G_{i})^{[l]},\bm{c}^{*}_{j}) \geq \text{sim}(\psi(G_{i})^{[l]},\bm{c}^{*}_{h}), $}
\\
\quad &\text{$\forall h \neq j$ and $\text{sim}(\psi(G_{i})^{[l]},\bm{c}^{*}_{j}) \geq \beta$}, 
\\
\mathbf{0}, & \text{sim}(\psi(G_{i})^{[l]},\bm{c}^{*}_{j}) < \beta,
\end{cases}
\end{gather}
where $\beta$ is a hyperparameter, $\text{sim}(\cdot)$ calculates the cosine similarity. $\mathcal{L}_{v}$ is used to optimize $O$, the rationality of which is demonstrated along with the proof of Theorem \ref{thm:leq}. Then, we leverage predictions made based on potential causes to guide our training process. Specifically, we propose the following loss function:
\begin{gather}
    \mathcal{L}_{g} =  \sum_{i=1}^{N} \mathcal{H}\Big( \frac{1}{|\mathcal{V}_i|}\sum_{l=1}^{|\mathcal{V}_i|} \bm{w}_{i,l}, \bar{\delta}\big(\psi(G_{i})\big)\Big).
\end{gather}
We add up $\mathcal{L}_{g}$ and $\mathcal{L}_{v}$, along with the $\mathcal{L}_{ce}$ and $\mathcal{L}_{o}$ adopted within the pre-training phase, to updated the model. The update procedure will last for $\lfloor \frac{\mu}{2} \rfloor$ epochs and the potential causes will be relocated until convergence or reach maximum epochs.

\section{Theoretical Analysis.}
In this chapter, we support our proposed method through theoretical analysis and proof. Based on the Definitions \ref{def:poc} concerning potential causes and the Definition \ref{def:graph_causal_subset} concerning graph causal subset $G^{*}$, we propose the following theorem:
\begin{theorem}
Given the $i$-th potential cause $\Gamma_{i}$  of $Y$ and the set of all potential causes $\bigcup_{i \in \mathcal{I}^p } \Gamma^{p}_{i}$ of $Y$, where $\mathcal{I}^p$ denotes the index set of the potential causes, it can be concluded that $G^{*} \subseteq \Big(\big( \bigcup_{i \in \mathcal{I}^{p} } \Gamma^{p}_{i} \big) \cap G \Big)$. Furthermore, if $\widetilde{Y}$ does not possess any causal relationships with $G$, $G^{*} = \Big(\big( \bigcup_{i \in \mathcal{I}^{p} } \Gamma^{p}_{i} \big) \cap G \Big)$.
\label{thm:lag}
\end{theorem}
The proof can be found in \textbf{Appendix} A.1. Theorem \ref{thm:lag} demonstrates that identifying the potential causes can effectively predict $G^{*}$, and as the causal association between $\widetilde{Y}$ and $G^{*}$ decreases, the intersection of $G$ and the potential causes will converge to $G^{*}$. Next, we analyze if our method accurately identifies potential causes, starting with this assumption:
\begin{assumption}
    Dataset $\{ G_{i}\}_{i=1}^{N}$ is large enough so that if there exist context $S$ that for some $\widetilde{G} \subset G$, $\widetilde{G} \perp \!\!\!\perp Y|S$ holds, then there is bound to be a subset $\mathcal{\hat{G}}$ of $ \{ G_{i}\}_{i=1}^{N} $ within which $\widetilde{G} \perp \!\!\!\perp Y$ holds.
\label{asp:perp}
\end{assumption}
The rationale behind Assumption \ref{asp:perp} is that as long as the dataset encompasses all cases, it is always possible to find data samples that validate $\widetilde{G} \perp \!\!\!\perp Y$. Even if such samples cannot be found, a relationship that exists across all samples and cases already serves as a sufficient basis for prediction.
\begin{theorem}
Suppose Assumption \ref{asp:perp} holds, for node representation $\bm{Z}$ and the corresponding graph data $G^{\bm{Z}} \subset G$ that lead to the output of $\bm{Z}$, given the following conditions : 1)$ \ \mathcal{L}_{o} $ that defined in Equation \ref{eq:lo} is minimized ; 2) $ \ \exists d, \bar{\delta}(\bm{Z})^{[d]} \neq 0$, then $G^{\bm{Z}}$ is a potential cause of $Y$. 
\label{thm:is_potential_cause}
\end{theorem}

The proof can be found in \textbf{Appendix} A.2. Theorem \ref{thm:is_potential_cause} demonstrates that as long as $ \ \exists d, \bar{\delta}(\bm{Z})^{[d]} \neq 0$ holds and the model has been sufficiently trained through $\mathcal{L}_{o} $, $G^{\bm{Z}}$ can be considered a potential cause. In our practical implementation, as the ideal condition may not be met, we have adopted a threshold $\Delta$ to manage the rigor of our judgments. Additionally, our evaluations are conducted across all samples, further ensuring accuracy.

\begin{table*}[ht]\tiny
	\setlength{\tabcolsep}{5pt}
 	\caption{Comparison of test accuracy (mean ± std) across seven datasets under random PLL label. Each experiment was conducted ten times to calculate the average result. Best records are highlighted in \textbf{bold}, while second-best records are marked with \underline{underline}.}
	\label{tab:nb}

	\begin{center}
		\begin{tabular}{l|ccccccc}
			\hline\rule{-1pt}{6pt}
			\multirow{2}*{Method}  &  Graph-SST5  & Graph-SST5 & {Graph-Twitter} & {Graph-Twitter} & \multirow{2}*{Graph-SST2} & \multirow{2}*{COLLAB} & REDDIT-  \\
			& (OOD) & (ID) & (OOD) & (ID) & & & MULTI-5K   \\ 	
			\hline\rule{-1pt}{6pt}
			\text{ARMA} \cite{arma} & {35.43$\pm$1.19} & {44.83$\pm$2.84} & {54.00$\pm$4.00} & {55.52$\pm$3.90} & {83.63$\pm$1.58} & {63.33$\pm$2.84} & {40.32$\pm$4.99}\\
			\text{DIR} \cite{DBLP:conf/iclr/WuWZ0C22} & {33.95$\pm$4.60} & {33.75$\pm$5.72} & {47.48$\pm$6.31} & {48.17$\pm$3.95} & {44.10$\pm$0.01} & {53.57$\pm$3.52} & {37.33$\pm$4.62}\\
			\text{CIGA} \cite{ciga} & {38.28$\pm$4.07} & {45.66$\pm$0.87} & {61.64$\pm$2.90} & {60.14$\pm$1.68} & {85.02$\pm$0.69} & {61.20$\pm$2.80} & {39.82$\pm$1.92}\\
			\text{DISC} \cite{fan2022debiasing} & {37.70$\pm$0.10} & {35.81$\pm$0.10} & {53.00±2.03} & {56.04$\pm$0.36} & {83.41$\pm$4.50} & {57.33$\pm$0.23} & {42.11$\pm$3.12}\\
			\text{PiCO} \cite{pico} & {39.69$\pm$3.65} & {45.84$\pm$0.93} & {60.17$\pm$0.22} & {51.14$\pm$0.67} & {30.89$\pm$0.36} & {48.34$\pm$0.12} & {35.05$\pm$0.59}\\
			\text{ML-PLL} \cite{mlpll} & {32.22$\pm$1.96} & {42.69$\pm$0.30} & {42.60$\pm$0.35} & {51.12$\pm$2.97} & \bf{87.29$\pm$1.07} & {51.87$\pm$3.35} & {36.66$\pm$0.42}\\
			\text{HTML} \cite{html} & {30.38$\pm$0.73} & {39.54$\pm$0.25} & {45.07$\pm$1.12} & {49.34$\pm$0.18} & {70.39$\pm$0.67} & {33.53$\pm$0.05} & {30.18$\pm$0.78}\\
            \hline\rule{-1pt}{6pt}
   			\text{GPCD-w/o-$\lambda$} & \underline{42.37$\pm$1.74} & \underline{47.37$\pm$0.21} & \underline{62.67$\pm$0.13} & \underline{61.05$\pm$0.07} & {86.20$\pm$0.10} & \underline{64.07$\pm$1.10} & \underline{45.51$\pm$0.57}\\
			\text{GPCD-w/o-auxiliary} & {42.31$\pm$0.57} & {47.33$\pm$0.24} & {61.81$\pm$0.27} & {60.61$\pm$0.08} & {86.86$\pm$0.19} & {63.67$\pm$0.31} & {45.18$\pm$0.31}\\
   			\text{GPCD} & \bf{44.69$\pm$2.11} & \bf{48.01$\pm$0.18} & \bf{63.97$\pm$1.33} & \bf{61.29$\pm$0.08} & \underline{87.12$\pm$0.16} & \bf{65.07$\pm$0.23} & \bf{49.10$\pm$0.35}\\

			\hline
		\end{tabular}
	\end{center}
        \label{table:res1}
\end{table*}

\begin{table*}[ht]\tiny
        \vspace{-7pt}
   \setlength{\tabcolsep}{5pt}
 	\caption{Test accuracy (mean±std) comparison on three datasets with competitive PLL label at different ambiguity levels $\rho$. Each experiment was conducted ten times to calculate the average result. The label semantically closest to the ground-truth is added to the candidate label set with a probability of $\rho$. Best records are highlighted in \textbf{bold}, while second-best records are marked with \underline{underline}.}
	\label{tab:nb}
	\begin{center}
		\begin{tabular}{l|cccccc}
			\hline\rule{-1pt}{6pt}
			
			\multirow{2}*{Method} &  Graph-SST5  & Graph-Twitter & COLLAB  &  Graph-SST5  & Graph-Twitter & COLLAB  \\
                & ($\rho=0.9$) & ($\rho=0.9$) & ($\rho=0.9$) & ($\rho=0.7$) & ($\rho=0.7$) & ($\rho=0.7$) \\
			\hline\rule{-1pt}{6pt}
			\text{ARMA} \cite{arma} & {39.83$\pm$0.70} & {46.69$\pm$0.88} & {60.27$\pm$3.72} & {42.64$\pm$0.65} & {47.41$\pm$5.20} & {59.67$\pm$0.07}\\
			\text{DIR} \cite{DBLP:conf/iclr/WuWZ0C22} & {38.98$\pm$0.30} & {45.15$\pm$1.16} & {57.10$\pm$0.17} & {38.95$\pm$1.77} & {44.19$\pm$5.70} & {50.17$\pm$1.15}\\
			\text{CIGA} \cite{ciga} & {41.65$\pm$1.32} & {57.22$\pm$1.34} & \underline{60.35$\pm$2.22} & {42.64$\pm$1.48} & {60.28$\pm$1.75} & {60.72$\pm$3.69}\\
			\text{DISC} \cite{DBLP:conf/nips/Fan0MST22} & {39.38$\pm$0.12} & {52.31$\pm$0.57} & {50.20$\pm$0.05} & {38.95$\pm$2.48} & {54.32$\pm$3.24} & {54.13$\pm$0.23}\\
			\text{PiCO} \cite{pico} & {41.93$\pm$0.19} & \underline{57.83$\pm$0.56} & {49.80$\pm$1.13} & \underline{43.65$\pm$0.86} & {55.79$\pm$0.64} & {37.09$\pm$1.78}\\
			\text{ML-PLL} \cite{mlpll} & {42.30$\pm$0.26} & {51.20$\pm$0.03} & {46.80$\pm$0.57} & {42.86$\pm$0.34} & {51.12$\pm$2.97} & {52.64$\pm$2.59}\\
			\text{HTML} \cite{html} & {37.26$\pm$0.77} & {48.63$\pm$1.01} & {30.74$\pm$1.05} & {39.34$\pm$0.11} & {50.75$\pm$0.72} & {31.49$\pm$0.78}\\      
            \hline\rule{-1.5pt}{6pt}
			\text{GPCD-w/o-$\lambda$} & \underline{42.72$\pm$0.08} & {57.25$\pm$0.18} & {60.33$\pm$0.28} & {43.48$\pm$0.13} & \underline{61.19$\pm$0.22} & {60.60$\pm$0.57}\\
			\text{GPCD-w/o-auxiliary} & {42.19$\pm$0.15} & {57.16$\pm$0.38} & {60.01$\pm$0.59} & {45.31$\pm$0.30} & {59.46$\pm$0.08} & \underline{61.53$\pm$0.81}\\
   			\text{GPCD} & \bf{43.03$\pm$0.29} & \bf{58.69$\pm$0.07} & \bf{61.52$\pm$0.97} & \bf{46.54$\pm$0.66} & \bf{61.31$\pm$0.18} & \bf{63.73$\pm$1.01}\\

			\hline
		\end{tabular}
	\end{center}
        \label{table:res2}
\end{table*}

\begin{figure*}[h]
	\centering
	\subfigure[Training: Positive sentiment]{
			\includegraphics[width=0.32\linewidth]{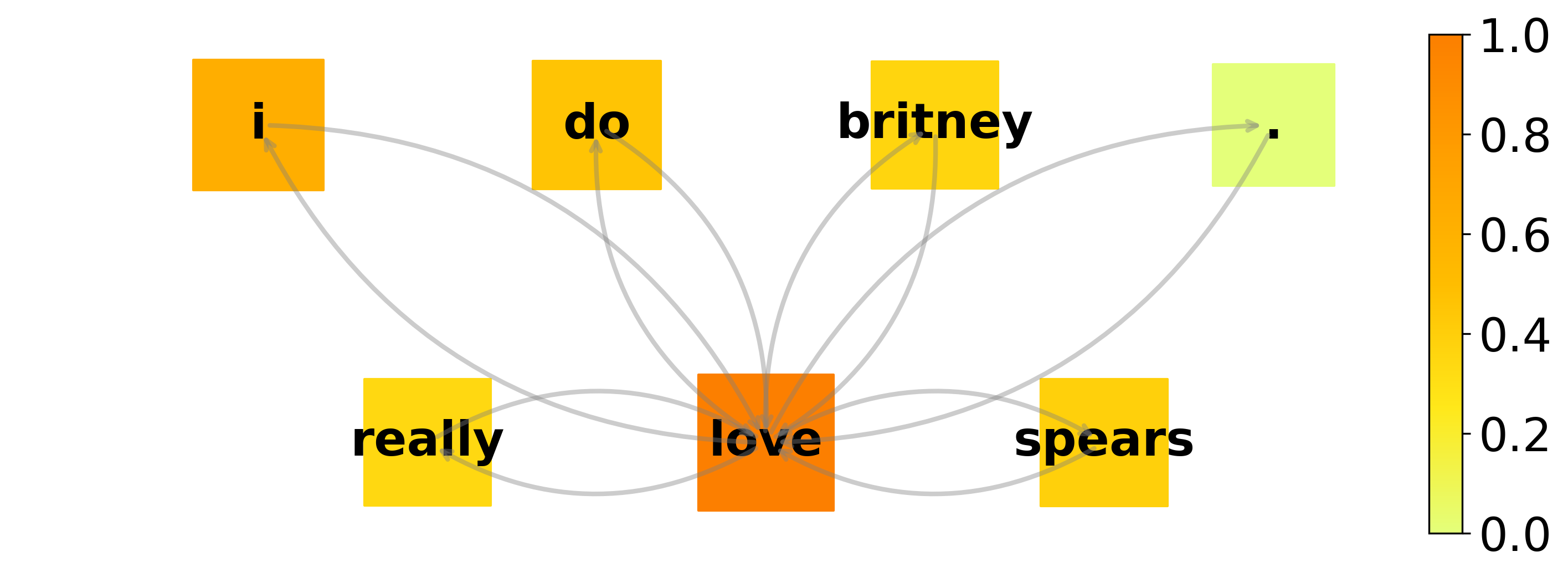}
	}%
	\subfigure[Training: Neutral sentiment]{
			\includegraphics[width=0.32\linewidth]{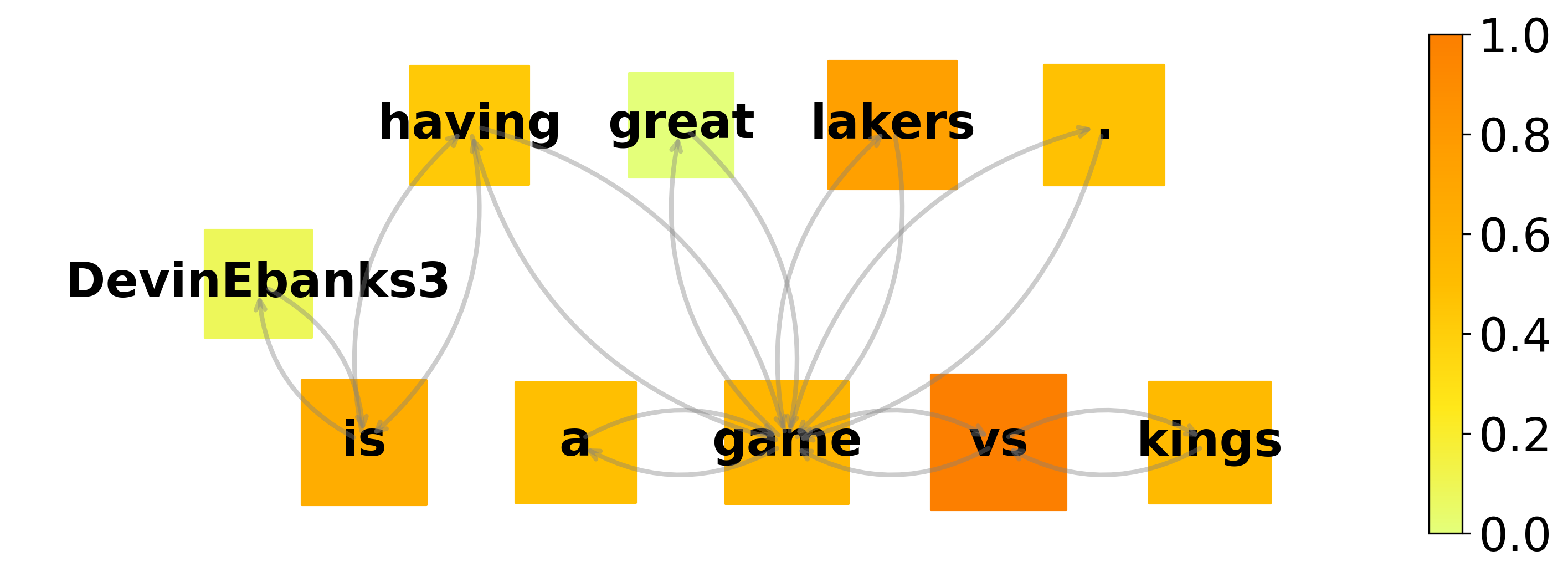}
	}%
	\subfigure[Training: Negative sentiment]{
			\includegraphics[width=0.32\linewidth]{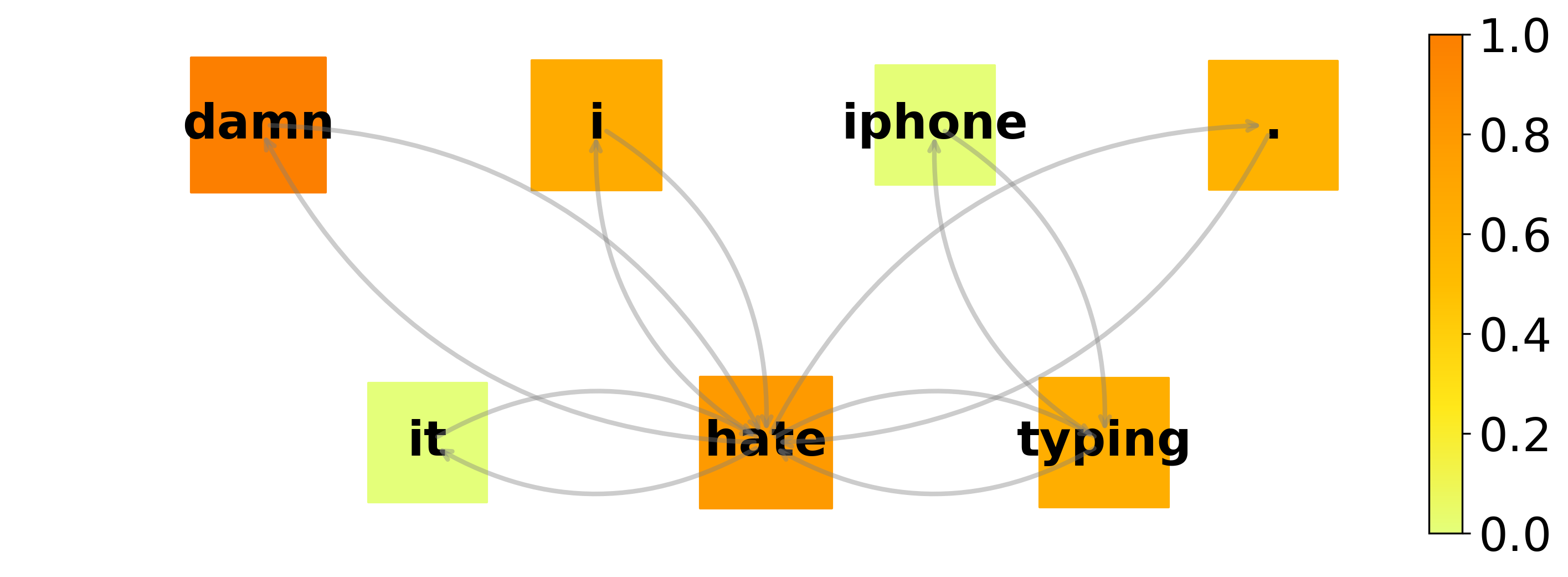}
	}%
        \vspace{-10pt}
	\subfigure[Testing: Positive sentiment]{
			\includegraphics[width=0.32\linewidth]{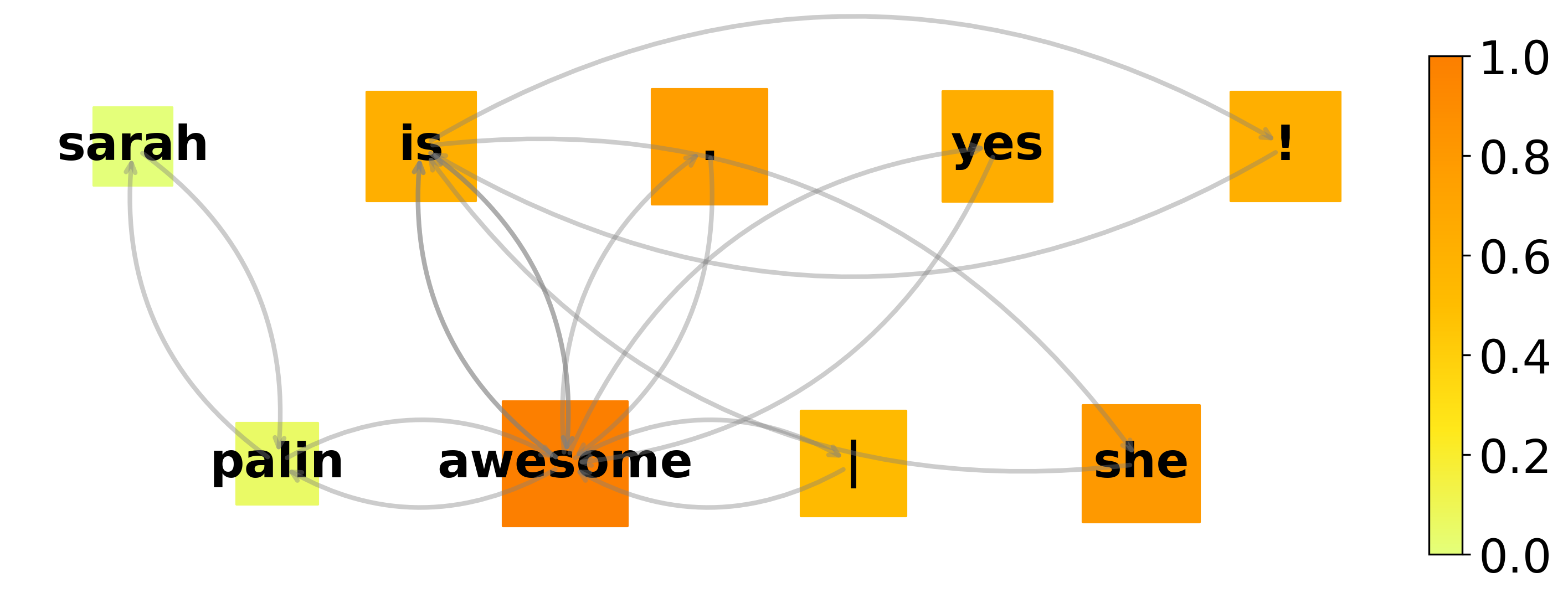}
	}%
	\subfigure[Testing: Neutral sentiment]{
			\includegraphics[width=0.32\linewidth]{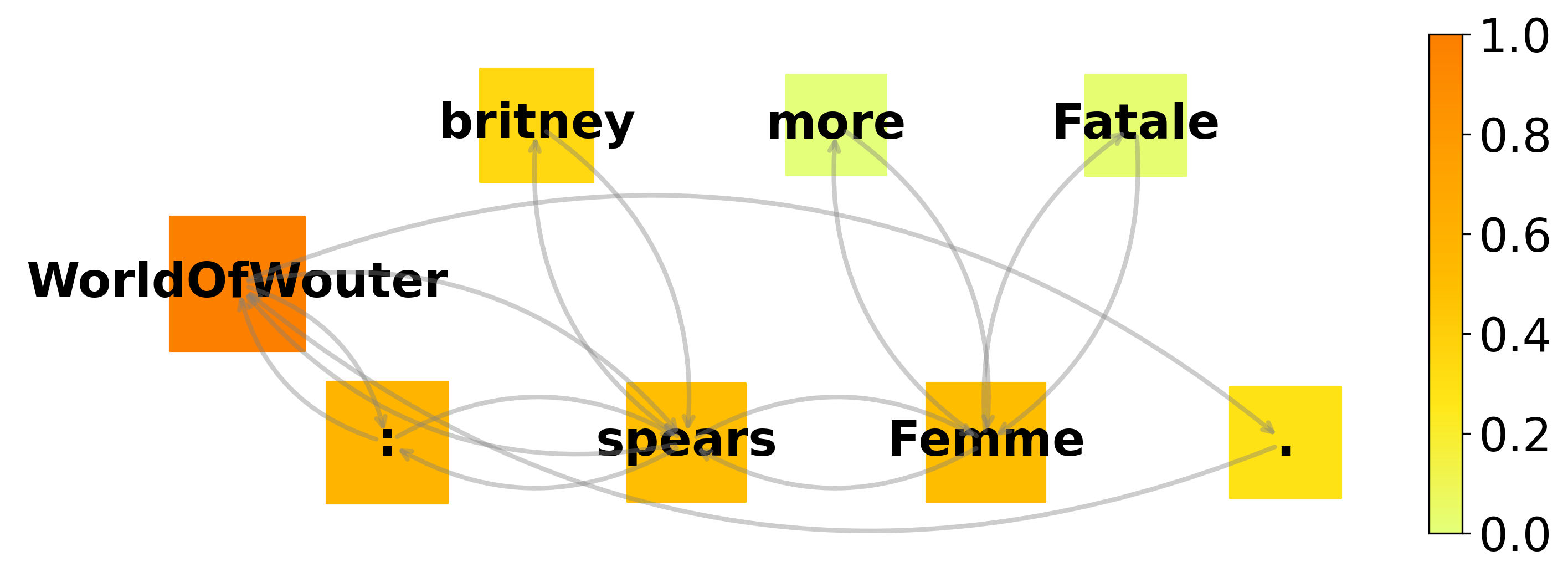}
	}%
	\subfigure[Testing: Negative sentiment]{
			\includegraphics[width=0.32\linewidth]{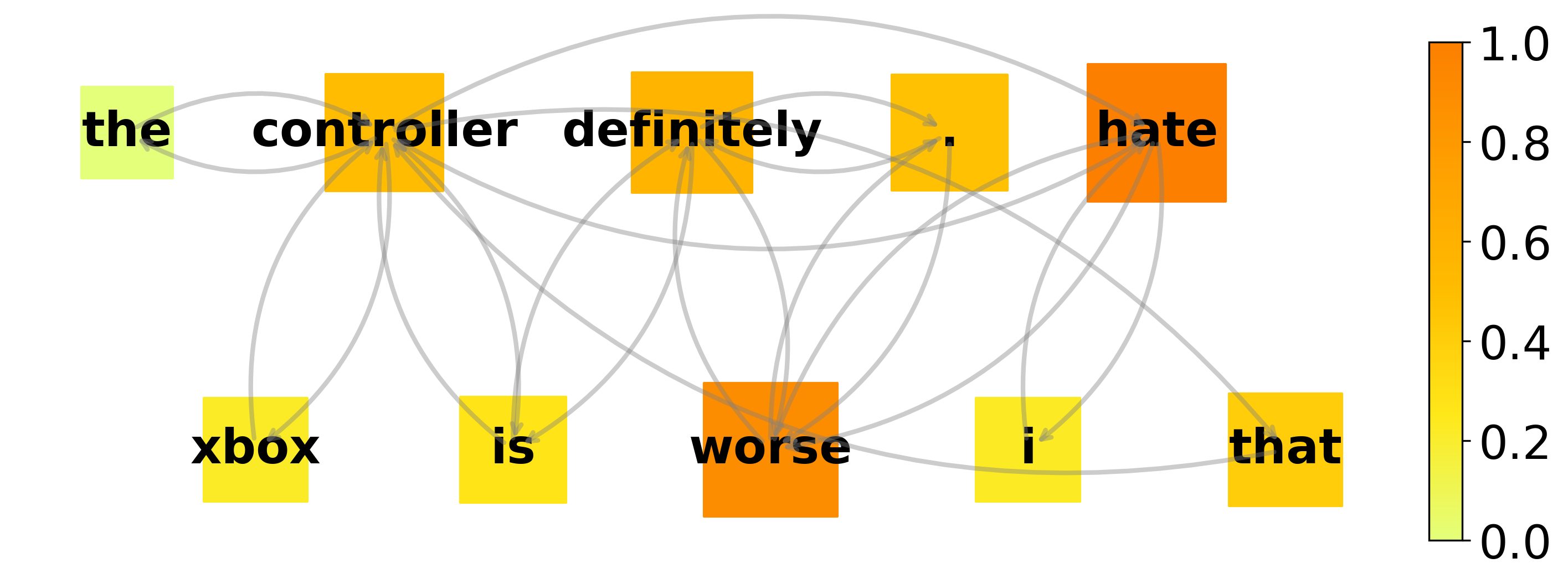}
	}%
 
	\centering
	\caption{Visualization of the node-level prediction vector on the Graph-Twitter dataset. Each graph represents a sentence. The subplot titles indicate the sentiment type corresponding to each sample.
}
	\label{fig:sentiment}
\end{figure*}

\begin{theorem} 
Given conditions that: 1) $\widetilde{Y}$ is not fully decided with $G^{*}$, i.e., $P(\widetilde{Y}=\widetilde{Y}_{i}|G^{*}=G_{i}^{*})\neq 1, \forall i \in \{1,2,...,N\}$; 2) $K = 2$, i.e., there are two labels attached with each sample, define a loss $\mathcal{L}_{1}$ as follows:
\begin{gather}
    \mathcal{L}_{1} = \sum_{i=1}^{N} \sum_{k=1}^{K} \mathcal{H}\Big( \text{epow}(f(G_{i}^{*}),\lambda), Y_{i}^{[k]}\Big), 
\end{gather}
where $\lambda > 1$, the form of $\text{epow}(\cdot)$ has already discussed below Equation \ref{eq:lv}. If $\mathcal{L}_{1}$ attains its respective minimum value, then we can conclude that the following loss $\mathcal{L}_{2}$ also reaches minimal:
\begin{gather}
    \mathcal{L}_{2} = \sum_{i=1}^{N} \mathcal{H}\Big( f(G_{i}^{*}), Y_{i}^{*}\Big), 
\end{gather}
i.e., the conventional cross-entropy loss that using only $Y^{*}$ as label also attains its minimum value.
\label{thm:leq}
\end{theorem}
The proof can be found in \textbf{Appendix} A.3. Theorem \ref{thm:leq} illustrates that discovering $G^{*}$ allows us to obtain a model equivalent to a GNN trained without noisy labels. Under specific conditions, Theorem \ref{thm:leq} holds true in a broader spectrum of scenarios. We demonstrate this through the following corollary.
\begin{corollary}
    Theorem \ref{thm:leq} will still hold for $K > 2$ if the parameter $\lambda$ within $\mathcal{L}_{1}$ can be set to an arbitrarily large value.
\label{cly:tt}
\end{corollary} 
The proof and required values for the parameter $\lambda$ can be found in \textbf{Appendix} A.4.

\section{Experiments}

\subsection{Comparsion with State-of-the-art methods}

\paragraph{Datasets.}

We created graph PLL datasets based on methods used in PLL of other domains. Specifically, we selected seven widely-recognized datasets: Graph-SST5 \cite{dataset:sst}, Graph-SST5 (OOD) \cite{dataset:sst,DBLP:conf/iclr/WuWZ0C22}, Graph-Twitter \cite{dataset:sst}, Graph-Twitter (OOD) \cite{dataset:sst,DBLP:conf/iclr/WuWZ0C22}, Graph-SST2 \cite{dataset:sst}, COLLAB \cite{dataset:collab} and REDDIT-MULTI-5K \cite{DBLP:conf/aaai/0001RFHLRG19}, then introduced random PLL label (extra PLL labels are randomly selected and added) and competitive PLL label (extra PLL labels are strongly correlated with the ground-truth label to induce confusion) using the method from \cite{mlpll}. This approach to PLL dataset construction follows the common paradigm in PLL researchs \cite{identify3,identify4,avg3}. Please refer to \textbf{Appendix} B.2 for details.

\begin{figure}
\vspace{-9pt}
	\centering
	\includegraphics[width=1\linewidth]{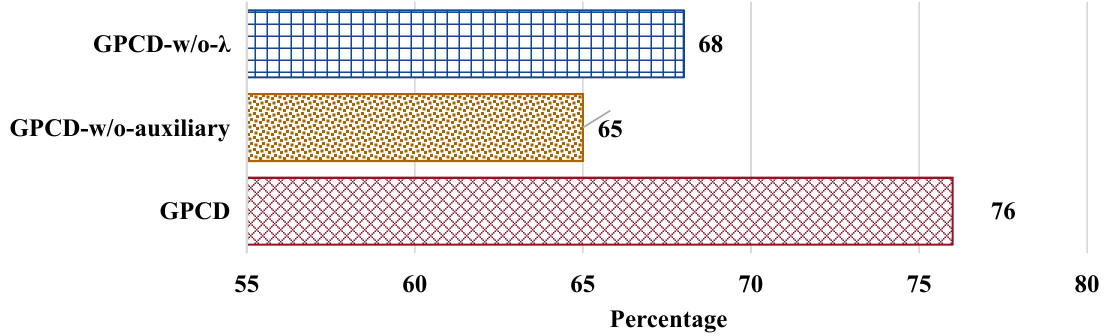}
    \vskip -0.1in
    \caption{Statistics of the focus ratio on sentiment words for different methods.}
    \label{fig:percentage}
    \vskip -0.1in
\end{figure}


\begin{figure}
	\centering
	\subfigure[ARMA(Train)]{
		\begin{minipage}[t]{0.27\linewidth}
			\centering
			\includegraphics[width=1\linewidth]{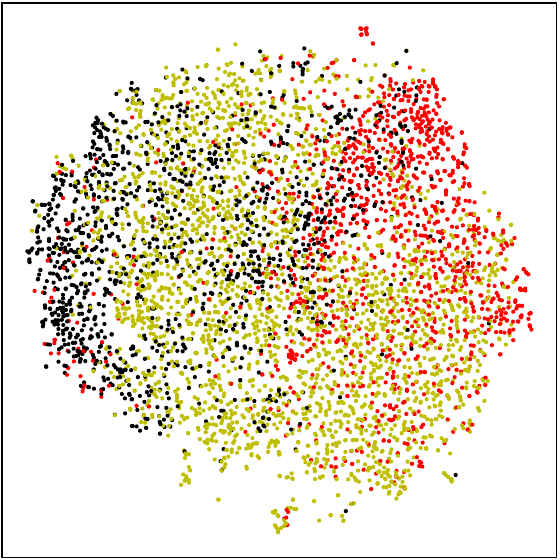}
		\end{minipage}%
	}%
 	\subfigure[ARMA(Test)]{
		\begin{minipage}[t]{0.27\linewidth}
			\centering
			\includegraphics[width=1\linewidth]{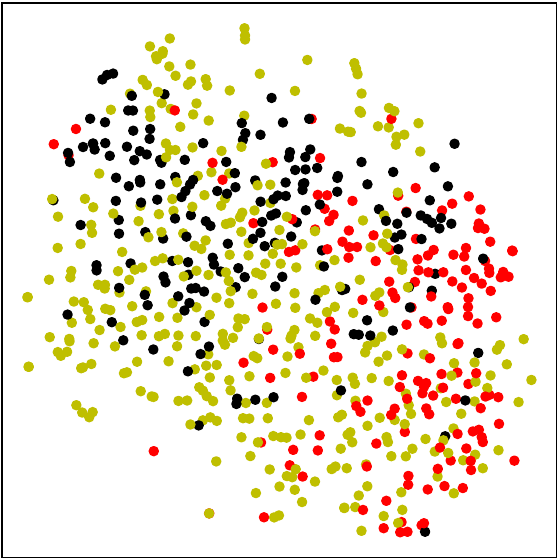}
		\end{minipage}%
	}%
	\subfigure[PICO(Train)]{
		\begin{minipage}[t]{0.27\linewidth}
			\centering
			\includegraphics[width=1\linewidth]{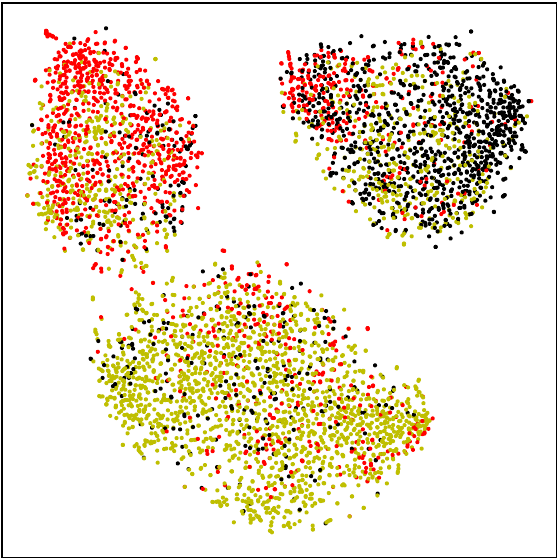}
		\end{minipage}%
	}%
    \\
 	\subfigure[PICO(Test)]{
		\begin{minipage}[t]{0.27\linewidth}
			\centering
			\includegraphics[width=1\linewidth]{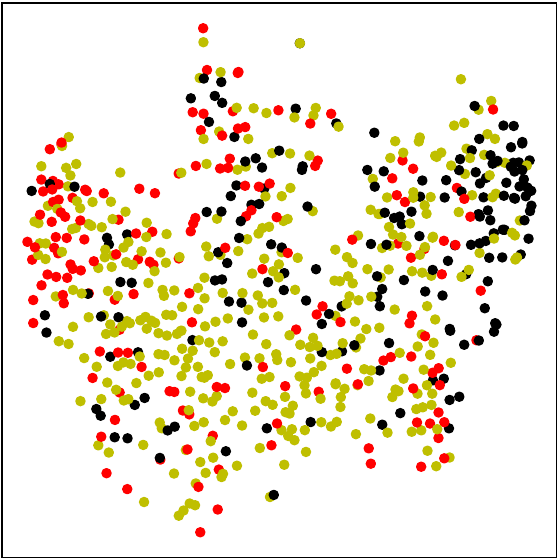}
		\end{minipage}%
	}%
	\subfigure[GPCD(Train)]{
		\begin{minipage}[t]{0.27\linewidth}
			\centering
			\includegraphics[width=1\linewidth]{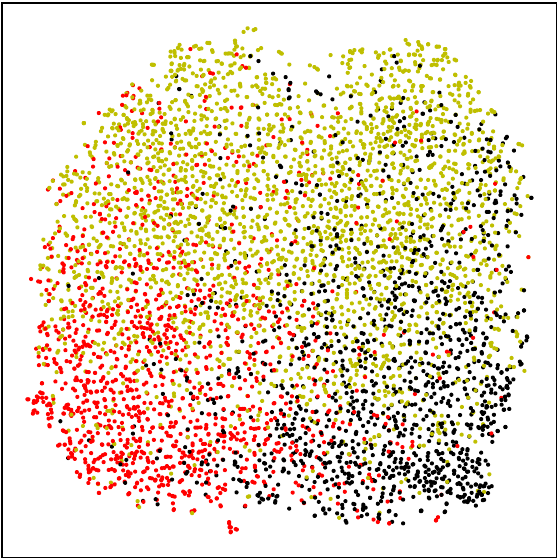}
		\end{minipage}%
	}%
 	\subfigure[GPCD(Test)]{
		\begin{minipage}[t]{0.27\linewidth}
			\centering
			\includegraphics[width=1\linewidth]{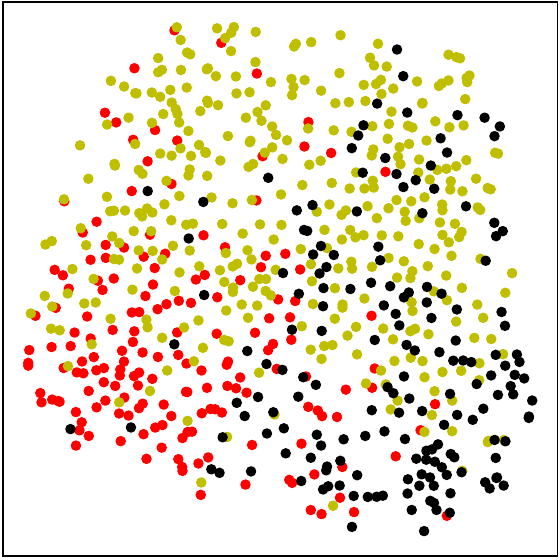}
		\end{minipage}%
	}%
	\centering

        \vskip -0.1in
 
	\caption{T-SNE visualization of the graph features on Graph-Twitter with random label noise. The first and second rows represent the features extracted from the training set and test set, respectively.}
	\label{fig:tsne}
\vspace{-10pt}
\end{figure}

\begin{figure}
	\centering
	\subfigure[Train]{
		\begin{minipage}[t]{0.45\linewidth}
			\centering
			\includegraphics[width=1\linewidth]{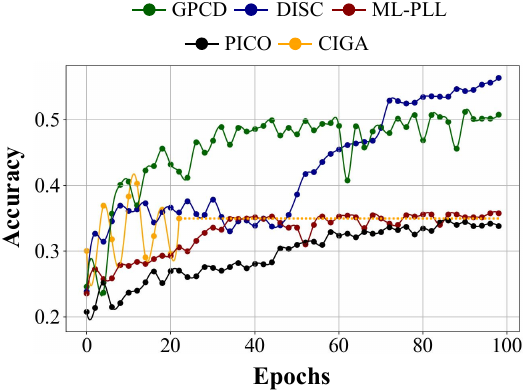}
        \vskip -0.1in
               \label{fig:conv-a}
		\end{minipage}%
	}%
	\subfigure[Test]{
		\begin{minipage}[t]{0.45\linewidth}
			\centering
			\includegraphics[width=1\linewidth]{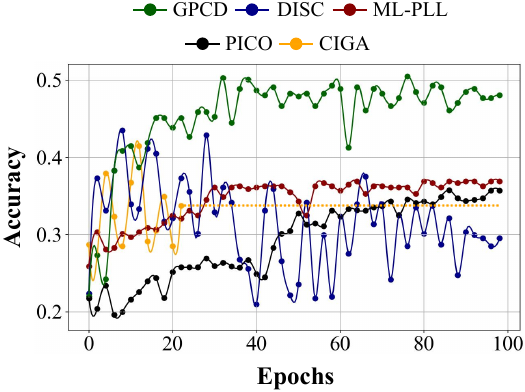}
        \vskip -0.1in
               \label{fig:conv-b}
		\end{minipage}%
	}%

	\centering
      \vskip -0.1in
 
	\caption{Comparison of GPCD's performance on the REDDIT-MULTI-5K dataset with other methods along the training procedure.}
     \vskip -0.1in

	\label{fig:conv}
\end{figure}


\paragraph{Settings.}
We compare GPCD with six baseline methods, including 1) graph causal learning methods: DIR \cite{DBLP:conf/iclr/WuWZ0C22}, CIGA \cite{ciga}, DISC \cite{DBLP:conf/nips/Fan0MST22}; 2) PLL methods: PICO \cite{pico}, ML-PLL \cite{mlpll}; 3) conventional GNN method ARMA \cite{arma}; 4) Graph representation learning method specifically designed for handling noisy data in labels, HTML \cite{html}. DEER \cite{10543125} that mentioned within the related works is not compared as none official implementation of this method is found. To ensure a fair comparison, we employed the same GNN backbone and dataset settings for all methods. Please refer to \textbf{Appendix} B. for details.
To validate the effectiveness of $\mathcal{L}_{v}$ and auxiliary training, we conducted two ablation experiments. For GPCD-w/o-$\lambda$, we set the $\lambda$ value to 1, and the function's concavity was removed according to Theorem \ref{thm:leq}. For GPCD-w/o-auxiliary, we eliminated the extraction of potential causes, and potential causes were no longer utilized as auxiliary guidance during the training process.

\paragraph{Results.}
Experimental results within Tables \ref{table:res1} and \ref{table:res2} show that GPCD outperforms baseline methods, demonstrates strong generalization on OOD data, and performs well with label noise. In the ablation experiments, GPCD's performance was superior to GPCD-w/o-$\lambda$ and GPCD-w/o-auxiliary, demonstrating the effectiveness of our designs.  We further discuss computational complexity within \textbf{Appendix} B.6.

\subsection{In-depth Study}

\paragraph{Analysis on Causality of Representations.}
In Figure \ref{fig:sentiment}, we visualize the absolute value of the output of $\bar{\delta}(\cdot)$ according to each node in the Graph-Twitter dataset. This highlights the words our model relies on for final predictions. It can be observed that GPCD focuses more on sentiment-related words. As Graph-Twitter is a sentiment analysis dataset, such a phenomenon demonstrates its ability to better capture causal information. We further randomly selected 100 samples from Graph-Twitter and calculated the ratio of sentiment-related words as the most focused nodes by GPCD and the ablation versions of it. The results are demonstrated within Figure \ref{fig:percentage}, which shows that GPCD more accurately identifies sentiment words.




\paragraph{Feature Visualization.}
Figure \ref{fig:tsne} visualizes graph features from the Graph-Twitter dataset with random label noise using t-SNE \cite{tsne}. We compare three methods: a) ARMA, which treats labels equally; b) PICO, a partial-label method; and c) our method, GPCD. ARMA's features are mixed in feature space. PICO forms three clusters on the training set, but the categories within each cluster are still mixed, and its features lack generalization, differing between training and test sets. In contrast, GPCD clusters features into three distinct groups on the training set, maintaining the same distribution on the test set. GPCD effectively captures features from the graph that are closely related to the ground-truth label.

\paragraph{Training Procure Analysis.}
In Figure \ref{fig:conv}, we conducted a comprehensive comparative analysis of the performance of various models throughout the entire training process. The results reveal that the GPCD model surpasses other models in terms of performance, exhibiting higher learning efficiency and enhanced stability. It is noteworthy that although DISC achieved better results in the subsequent learning phase on the training set, only GPCD demonstrates optimal performance on the test set. This substantiates that our method not only eliminates interference information but also consistently learns stably.


\section{Conclusion}
We introduce a novel method, GPCD, with the goal of enabling GNN models to effectively learn discriminative information in the context of Partially Labeled Learning. GPCD achieves this by effectively combating label noise in the data by identifying potential causes within the graph data. We provide theoretical analysis to substantiate the rationale behind GPCD's design. Additionally, multiple experiments demonstrate the superiority of the GPCD method.


\bibliography{aaai25}

\newpage
\appendix
\onecolumn
\clearpage
\section{A. Proofs}

\subsection{A.1. Proof of Theorem \ref{thm:lag}}
\label{prf:lag}
In order to prove the theorem, our initial step involves showing that $G^{*} \subseteq \Big(\big( \bigcup_{i \in \mathcal{I}^{p}} \Gamma^{p}_{i} \big) \cap G \Big)$. 
To achieve such a goal, it suffices to demonstrate that any 
$\Gamma^{p}_{i} \in G^{*}$ is a potential cause of $Y$. As delineated in Definition \ref{def:graph_causal_subset}, any subset of $G^{*}$ holds causal relationship with $Y$ and, therefore, is a genuine cause \cite{pearl2000models} of $Y$. Thus, According to the definition of genuine cause, $\Gamma^{p}_{i}$ is bound to be a potential cause of $Y$. Hence, we have substantiated the assertion that $G^{*} \subseteq \Big(\big( \bigcup_{i \in \mathcal{I}^{p}} \Gamma^{p}_{i} \big) \cap G \Big)$ is accurate.

Next, we will prove that if $\widetilde{Y}$ does not possess any causal relationships with $G^{*}$, $G^{*} = \Big(\big( \bigcup_{i \in \mathcal{I}^{p}} \Gamma^{p}_{i} \big) \cap G \Big)$. We have already demonstrated that $G^{*} \subseteq \Big(\big( \bigcup_{i \in \mathcal{I}^{p}} \Gamma^{p}_{i} \big) \cap G \Big)$ holds; thus, we only need to establish that for the case when there exist no relationships between $\widetilde{Y}$ and $G$, any $\Gamma^{p}_{i} \in \Big(\big( \bigcup_{i \in \mathcal{I}^{p}} \Gamma^{p}_{i} \big) \cap G \Big)$ satisfies $\Gamma^{p}_{i} \in G^{*}$.

We first propose that if $\widetilde{Y}$ does not possess any causal relationships with $G$, then the following equation holds:
\begin{equation}
    \Gamma^{p}_{j} \perp \!\!\!\perp \widetilde{Y}, \forall \Gamma^{p}_{j} \in  \Big(\big( \bigcup_{i \in \mathcal{I}^{p}} \Gamma^{p}_{i} \big) \cap G \Big).
\label{eq:gammaindwy}
\end{equation}
If not, then there will exist certain $\Gamma^{p}_{j}$ that satisfied: 1) $\Gamma^{p}_{j} \not \perp \!\!\!\perp \widetilde{Y} | \varnothing $, and 2) $\Gamma^{p}_{j} \perp \!\!\!\perp \widetilde{Y} | G$. Then, there will also exist $\Gamma^{p}_{j'}$ that satisfied: 1) $\Gamma^{p}_{j'} \not \perp \!\!\!\perp \widetilde{Y} | \varnothing $, and 2) $\Gamma^{p}_{j'} \perp \!\!\!\perp \widetilde{Y} | \hat{G}$, where $\hat{G} \subseteq G$ and the data within $\hat{G}$ recorded before $\Gamma^{p}_{j'}$, otherwise we can switch $\Gamma^{p}_{j'}$ with the element within $\hat{G}$ that recorded after $\Gamma^{p}_{j'}$. Then, according to genuine causation with temporal information \cite{pearl2000models}, $\Gamma^{p}_{j}$ is a genuine causation of $\widetilde{Y}$, which is contrary to $\widetilde{Y}$ does not possess any causal relationships with $G$. Therefore, under the given condition, Equation \ref{eq:gammaindwy} holds.

Subsequently, as $\Gamma^{p}_{j}$ is the potential cause of $Y$, therefore we have the following equation holds for all contexts:
\begin{equation}
    \Gamma^{p}_{j} \not \perp \!\!\!\perp Y, \forall \Gamma^{p}_{j} \in  \Big(\big( \bigcup_{i \in \mathcal{I}^{p}} \Gamma^{p}_{i} \big) \cap G \Big).
\label{eq:gammanindy}
\end{equation}
Here, context, as already explained within Definition \ref{def:poc}, means a set of variables tied to specific values. Based on Equation \ref{eq:gammaindwy} and \ref{eq:gammanindy}, and $Y = \{Y^{*},\widetilde{Y}\}$, we can conclude that the following equation holds for all context:
\begin{equation}
    \Gamma^{p}_{j} \not \perp \!\!\!\perp Y^{*}, \forall \Gamma^{p}_{j} \in  \Big(\big( \bigcup_{i \in \mathcal{I}^{p}} \Gamma^{p}_{i} \big) \cap G \Big).
\label{eq:gammanindy}
\end{equation}
As label $Y^{*}$ are created based on $G$, based on the temporal sequence causal relationship follows \cite{DBLP:conf/uai/PearlR95}, 
$Y^{*}$ won't hold any causal influence on $\Big(\big( \bigcup_{i \in \mathcal{I}^{p}} \Gamma^{p}_{i} \big) \cap G \Big)$. Furthermore, for $\text{PA}(\Gamma^{p}_{i}) \notin G$, i.e., variable that holds causal relationship upon $\Gamma^{p}_{i}$ yet excluded from $G$, we can pick a context $S$ that block all causal roots from $\text{PA}(\Gamma^{p}_{i})$ to $Y^{*}$ other than those went trough $\Gamma^{p}_{i}$, then have $\text{PA}(\Gamma^{p}_{i}) \not \perp \!\!\!\perp Y^{*} | S$ and $\text{PA}(\Gamma^{p}_{i})  \perp \!\!\!\perp Y^{*} | S \cap \Gamma^{p}_{i}$ holds. According to genuine causation with temporal information, $\Gamma^{p}_{i}$ holds causal influence upon $Y^{*}$, therefore $\Gamma^{p}_{i} \in G^{*}$, the theorem is proved.

\subsection{A.2. Proof of Theorem \ref{thm:is_potential_cause}}
\label{prf:is_potential_cause}


The proof of the theorem requires a more detailed analysis of the causal model corresponding to the graph data and labels in the PLL scenario. Therefore, we introduce the Structural Causal Model (SCM)\cite{pearl2000models} to help characterize this causal model.

\paragraph{Structural Causal Model (SCM)\cite{pearl2000models}.} A Structural Causal Model (SCM) is a mathematical framework used to represent and analyze causal relationships within a system. It typically consists of a set of variables and equations that describe how these variables influence each other. SCMs are particularly useful for distinguishing between correlation and causation, allowing researchers to make predictions about the effects of interventions. In an SCM, each variable is either an exogenous variable, which is determined by factors outside the model, or an endogenous variable, whose value is derived within the model based on other variables. The relationships are expressed through structural equations that link each endogenous variable to a function of other variables and possibly some exogenous variables.

In Figure \ref{fig:SCM}, we demonstrated the SCM for the graph data and labels in the PLL scenario. $\bar{\text{PA}}(\cdot)$ denotes the parents of a variable outside $G$. Set $\{\Omega_{i}\}_{i=1}^{r}$ denotes $G^{*}$, with $\Omega_{i}$ denotes a variable within $G^{*}$. Similarly, Set $\{\Omega_{j}\}_{j=r+1}^{r+m}$ denotes $G \setminus G^{*}$. $N$ denotes the noise variables outside of $G$ that may influence $\widetilde{Y}$. Solid arrows represent causal relationships. Double dashed arrows denote confounding arcs, which indicate uncertainty about the existence of causal relationships between variables, and if such causal relationships exist, the direction of causation is unknown. Single dashed arrows represent confounding arcs where the direction of influence is determinable. In the Figure, confounding arcs exist between all elements within set $\{\bar{\text{PA}}(\Omega_{i})\}_{i=1}^{r+m}$ and similarly within set $\{\Omega_{\gamma}\}_{\gamma=1}^{r+m}$. However, it is not feasible to depict all these confounding arcs in the diagram. For simplicity, relationships involving omitted variables, as well as those requiring connections across ellipses in the figure, are not illustrated. Based on Figure \ref{fig:SCM}, we propose the following lemma. 

\begin{figure*}
	\centering
	\begin{minipage}[t]{1\linewidth}
        \centering
        \includegraphics[width=0.7\linewidth]{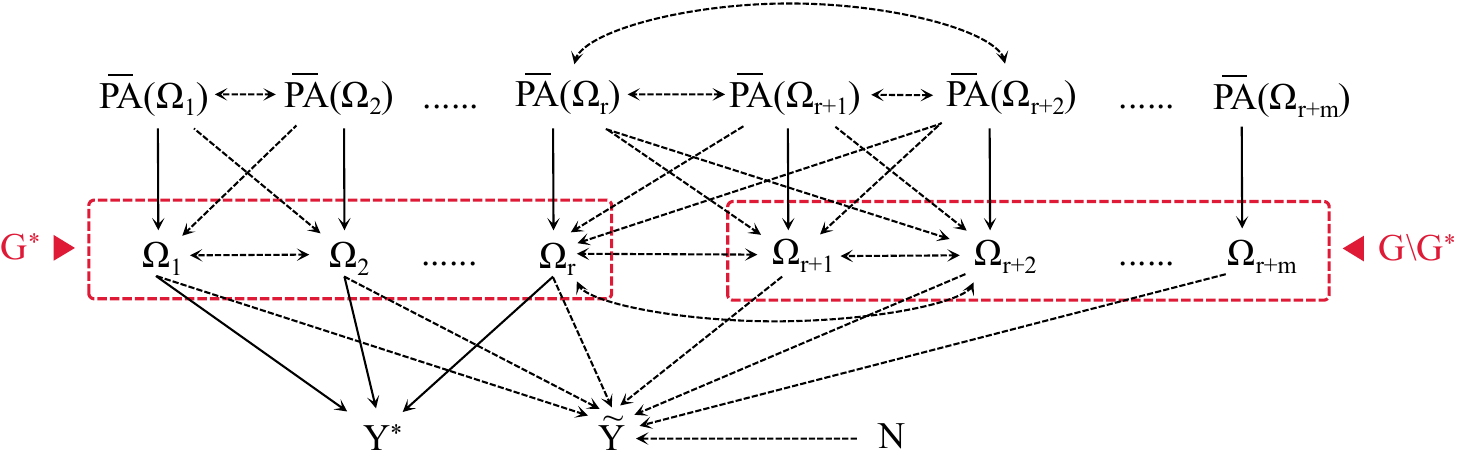}
	\end{minipage}
	\centering
    \vskip -0.05in
	\caption{SCM of the studied scene. For simplicity, relationships involving omitted variables, as well as those requiring connections across ellipses in the figure, are not represented.}
	\label{fig:SCM}
\end{figure*}

\begin{figure*}
	\centering
	\subfigure[]{
		\begin{minipage}[t]{0.7\linewidth}
			\centering
			\includegraphics[width=1\linewidth]{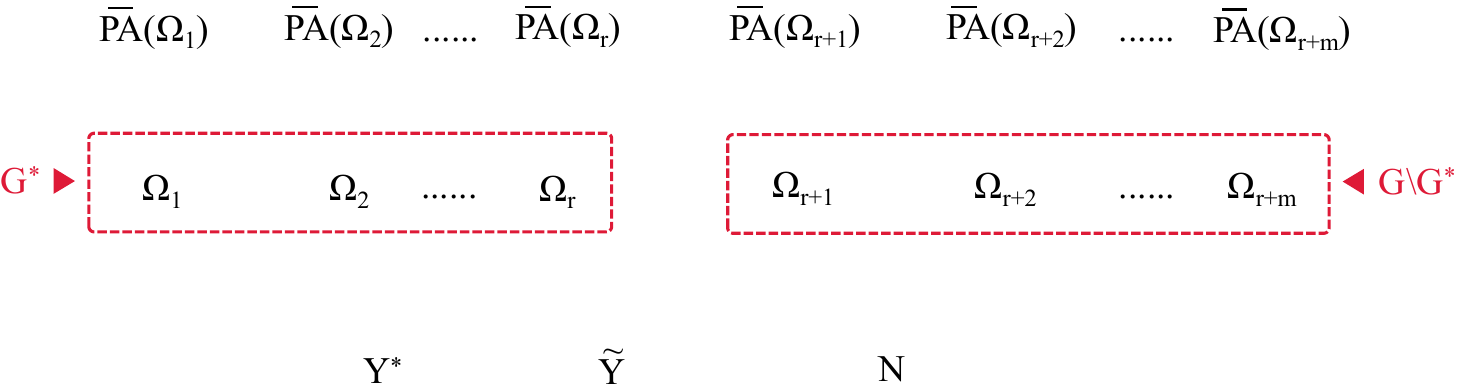}
   \vskip -0.05in
   \label{fig:IC1}
		\end{minipage}%
	}\\%
	\subfigure[]{
		\begin{minipage}[t]{0.7\linewidth}
			\centering
			\includegraphics[width=1\linewidth]{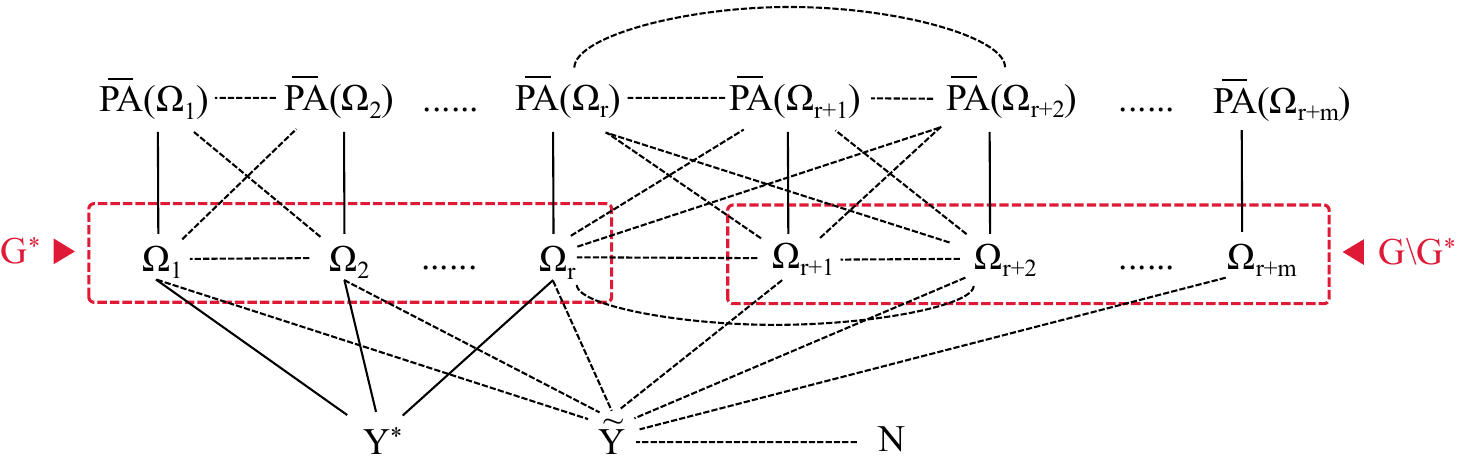}
   \vskip -0.05in
   \label{fig:IC2}
		\end{minipage}%
	}\\%
 	\subfigure[]{
		\begin{minipage}[t]{0.7\linewidth}
			\centering
			\includegraphics[width=1\linewidth]{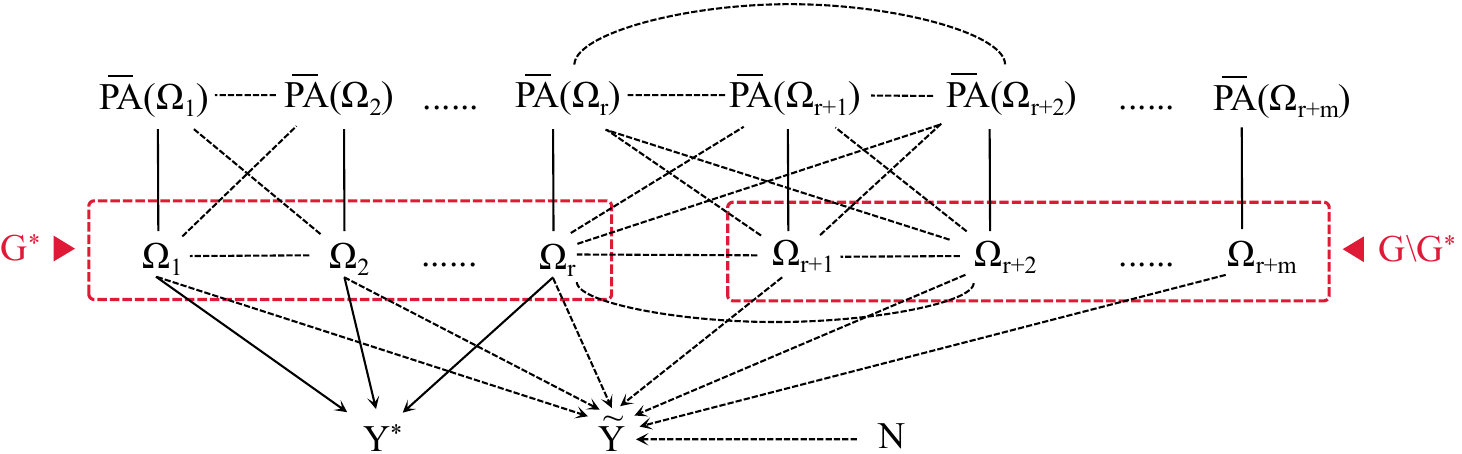}
   \vskip -0.05in
   \label{fig:IC3}
		\end{minipage}%
	}\\%
 	\subfigure[]{
		\begin{minipage}[t]{0.7\linewidth}
			\centering
			\includegraphics[width=1\linewidth]{./fig/SCM4.pdf}
   \label{fig:IC4}
		\end{minipage}%
	}%
	\centering
    \vskip -0.1in
	\caption{The procedure of IC algorithm.}

	\label{fig:IC}
\end{figure*}

\begin{lemma}
    The SCM within Figure \ref{fig:SCM} can represent the causal model in graph datasets under the PLL scenario.
\label{lma:SCM}
\end{lemma}

\begin{proof}
We utilized the IC algorithm \cite{DBLP:conf/uai/VermaP90} to demonstrate the validity of the proposed SCM. The IC algorithm is an effective method for deriving the structure of an SCM, and it consists of three steps. For step 1, the IC algorithm studies each pair of variables and identifies a set $S$ such that the pair of variables is conditionally independent given $S$. Then, the algorithm constructs an undirected graph where vertices are connected by an edge if and only if no such set $S$ can be identified. For step 2, the algorithm studies each pair of nonadjacent variables that share a common neighbor $c$, check the specified condition $c \in S$: 1) If the condition holds, proceed without any changes; 2) If the condition does not hold, modify the graph by adding arrowheads pointing towards $c$. For step 3, the algorithm studies the resulting partially directed graph, orienting the remaining undirected edges as much as possible while adhering to two constraints:1) Any alternative orientation should not result in the formation of a new v-structure. 2) Any alternative orientation should not create a directed cycle. Next, we will follow these steps to construct the proposed SCM from scratch.

\textbf{Step 1.} We first construct Figure \ref{fig:IC1} that represents all variables, then proceed to establish undirected edges. As the relationships within Set $\{\bar{\text{PA}}(\Omega_{i})\}_{i=1}^{r+m}$ are uncertain, it is impossible to find context $S$ such that $\bar{\text{PA}}(\Omega_{i}) \perp \!\!\!\perp \bar{\text{PA}}(\Omega_{j}) $ holds, $i$ and $j$ can be any possible index. Therefore we need to link all the variables within $\{\bar{\text{PA}}(\Omega_{i})\}_{i=1}^{r+m}$. The same condition holds for $\{\Omega_{\gamma}\}_{\gamma=1}^{r+m}$ and $\{\Omega_{\gamma}\}_{\gamma=1}^{r+m} \cap \{\bar{\text{PA}}(\Omega_{i})\}_{i=1}^{r+m}$, therefore these variables are linked. We link $\widetilde{Y}$ with $\{\Omega_{\gamma}\}_{\gamma=1}^{r+m}$ and $N$ for the same reason. Furthermore, as each $\Omega_{i}$ within $\{\Omega_{i}\}_{i=1}^{r}$ are part of $G^{*}$, therefore is the potential cause of $Y^{*}$, so we can't find any context $S$ that makes $\Omega_{i} \perp \!\!\!\perp Y^{*}$ holds, $i$ can be any index of $\Omega$ within $\{\Omega_{i}\}_{i=1}^{r}$, then we link the corresponding variables. As $\{\bar{\text{PA}}(\Omega_{i})\}_{i=r+1}^{r+m}  \perp \!\!\!\perp Y^{*}| G$ and $\{\bar{\text{PA}}(\Omega_{i})\}_{i=1}^{r+m}  \perp \!\!\!\perp N| G$ holds, these variables will not be linked. The results are demonstrated in Figure \ref{fig:IC2}. Please note that edges involving omitted variables, as well as those requiring connections across ellipses in the figure, are not illustrated. 

\textbf{Step 2.} Similar to the final step previously, we connect variables due to the indeterminate nature of their causal relationships. Consequently, we substitute these connections with double-dashed arrows to denote the uncertainty. Next, as within $\{\Omega_{i}\}_{i=1}^{r}$, if for certain context $S$, $\Omega_{1} \perp \!\!\!\perp \Omega_{r}|S$ holds, then $Y^{*} \notin S$. Therefore, we create directed edges $\Omega_{1} \rightarrow Y^{*} \leftarrow \Omega_{r}$. Similarly, we create $\Omega_{2} \rightarrow Y^{*}, \Omega_{3} \rightarrow Y^{*}, ..., \Omega_{r-1} \rightarrow Y^{*}$. Based on the same approach, we create $\Omega_{r+1} \rightarrow \widetilde{Y}, \Omega_{r+2} \rightarrow \widetilde{Y}, ..., \Omega_{r+m} \rightarrow \widetilde{Y}$. However, the existence of causal relationships between these variables is, as elaborated before, uncertain. Therefore, we adopt dashed single-head arrows to link them. The results are demonstrated in Figure \ref{fig:IC3}.

\textbf{Step 3.}
We construct $\bar{\text{PA}}(\Omega_{1}) \rightarrow \Omega_{1} \leftarrow \bar{\text{PA}}(\Omega_{2}) $, as it will generate new v-structures. Similarly, we could directed all edges between set $\{\bar{\text{PA}}(\Omega_{i})\}_{i=1}^{r+m}$ and set $\{\Omega_{i})\}_{i=1}^{r+m}$. As the existence of causal relationships between some of these variables is uncertain, we adopt dashed single-head arrows to link them. The rest will be linked with dashed double-head arrows. The results are illustrated within Figure \ref{fig:IC4}, which is identical with the proposed SCM in Figure \ref{fig:SCM}, the lemma is proved.
\end{proof}

Drawing on Lemma \ref{lma:SCM} and the proposed SCM, we demonstrate the theorem by showing that under the circumstances of the theorem, if the condition $\exists d, \delta(\bm{Z})^{[d]} \neq 0$ is satisfied, then $G^{\bm{Z}}$ is related to $Y$ in all contexts. According to the given conditions, $\mathcal{L}_{o}$ is maximized. Furthermore, we have:
\begin{align}
    \mathcal{L}_{o} &= -\frac{1}{NK} \sum_{i=1}^{N} \sum_{k=1}
    ^{K} \mathcal{H}\Big( \text{norm}\Big(\sum_{l=1}^{|\mathcal{V}_i|} \Big(M_{i} \circ \bar{\delta}\big(\psi(G_{i})\big)\Big)^{[l]}\Big), Y_{i}^{[k]}\Big)  \nonumber\\
    &=
    -\frac{1}{NK}\sum_{i=1}^{N} \sum_{k=1}^{K} \sum_{d=1}^{D} Y_{i}^{[k,d]} log \bigg( \text{norm}\Big(\sum_{l=1}^{|\mathcal{V}_i|} \Big(M_{i} \circ \bar{\delta}\big(\psi(G_{i})\big)\Big)^{[l]}\Big)^{[d]}\bigg),
    \label{eq:loi}
\end{align}
where $D$, as mentioned before, denotes the number of classes. Based on the SCM within Figure \ref{fig:SCM}, we have:
\begin{gather}
    -\frac{1}{NK}\sum_{i=1}^{N} \sum_{k=1}^{K} \sum_{d=1}^{D} Y_{i}^{[k,d]} log \bigg( \text{norm}\Big(\sum_{l=1}^{|\mathcal{V}_i|} \Big(M_{i} \circ \bar{\delta}\big(\psi(G_{i})\big)\Big)^{[l]}\Big)^{[d]}\bigg)  \nonumber\\
    = 
    -\frac{1}{NK}\sum_{i=1}^{N} \sum_{k=1}^{K} \sum_{d=1}^{D} Y_{i}^{[k,d]} log \bigg( \text{norm}\Big(\sum_{l=1}^{|\mathcal{V}_i|} \Big(M_{i} \circ \bar{\delta}\big(\psi(\{\Omega_{\gamma,i}\}_{\gamma=1}^{r+m})\big)\Big)^{[l]}\Big)^{[d]}\bigg),
    \label{eq:nkmsum}
\end{gather}
where $\{\Omega_{\gamma,i}\}_{\gamma=1}^{r+m} = G_{i}$, subscript $\gamma$ denote the index of different variable within $G$, subscript $i$ denotes the index of graph sample. Furthermore, we denote that $G^{\bm{Z}} = \{\Omega_{\beta}\}_{\beta=1}^{B}$, where $\{\Omega_{\beta}\}_{\beta=1}^{B} \subseteq \{\Omega_{\gamma}\}_{\gamma=1}^{r+m}$. $l^{\bm{Z}}$, similar to the preceding definitions, denotes the node index where the corresponding output representation is $\bm{Z}$. 

To conduct proof, we bring up a hypothesize that if condition $\delta(\bm{Z})^{[d]} \neq 0$ is satisfied, and $G^{\bm{Z}} = \{\Omega_{\beta}\}_{\beta=1}^{B}$ is not related to $Y^{[k,d]}$ in some contexts. Then we have $\{\Omega_{\beta}\}_{\beta=1}^{B}$ is not related to $Y$ in some context, i.e., in certain context $S$, $\{\Omega_{\beta}\}_{\beta=1}^{B} \perp \!\!\!\perp Y^{[k,d]}$. Based on Equation \ref{eq:loi} and \ref{eq:nkmsum}, we have:
\begin{align}
    \mathcal{L}_{o} &= -\frac{1}{NK}\sum_{i=1}^{N} \sum_{k=1}^{K} \sum_{d=1}^{D} Y_{i}^{[k,d]} log \bigg( \text{norm}\Big(\sum_{l=1}^{|\mathcal{V}_i|} \Big(M_{i} \circ \bar{\delta}\big(\psi(\{\Omega_{\gamma,i}\}_{\gamma=1}^{r+m})\big)\Big)^{[l]}\Big)^{[d]}\bigg) \nonumber\\
    &= -\frac{1}{NK}\sum_{i=1}^{N} \sum_{k=1}^{K} \sum_{d=1}^{D} Y_{i}^{[k,d]} log \bigg( \text{norm}\Big(\sum_{l=1}^{|\mathcal{V}_i|} \Big(M_{i} \circ \bar{\delta}\big(\psi\big(\{\Omega_{\beta,i}\}_{\beta=1}^{B} \cup (\{\Omega_{\gamma,i}\}_{\gamma=1}^{r+m} \nonumber\\
    & \quad \ \ \ \ \  \setminus \{\Omega_{\beta,i}\}_{\beta=1}^{B})\big)\big)\Big)^{[l]}\Big)^{[d]}\bigg) \nonumber\\
    &= -\frac{1}{NK}\sum_{i=1}^{N} \sum_{k=1}^{K} \sum_{d=1}^{D} Y_{i}^{[k,d]} log \bigg( \text{norm}\Big(\sum_{a=1}^{|\mathcal{A}_i|} \Big(M_{i} \circ \bar{\delta}\big(\psi\big(\{\Omega_{\beta,i}\}_{\beta=1}^{B} \cup (\{\Omega_{\gamma,i}\}_{\gamma=1}^{r+m}  \nonumber\\
    & \quad \ \ \ \ \  \setminus \{\Omega_{\beta,i}\}_{\beta=1}^{B})\big)\big)\Big)^{[a]} +\Xi_{i}\Big)^{[d]}\bigg), \nonumber\\
    \label{eq:adda}
\end{align}
where $\Xi_{i}$ denotes the sum of output values other than $\{(M_{i} \circ \bar{\delta}(\{\Omega_{\beta,i}\}_{\beta=1}^{B} \cup (\{\Omega_{\gamma,i}\}_{\gamma=1}^{r+m} \setminus \{\Omega_{\beta,i}\}_{\beta=1}^{B}))\}_{a=1}^{A_{i}}$, $\mathcal{A}_{i}$ denotes the set of nodes corresponding to $\{\Omega_{\beta}\}_{\beta=1}^{B}$ within graph sample $G_{i}$, furthermore, the following equation define the output of $\psi(\{\Omega_{\beta,i}\}_{\beta=1}^{B} )^{[a]}$:
\begin{gather}
    \psi(\{\Omega_{\beta,i}\}_{\beta=1}^{B} )^{[a]} = X, \forall a \in [1,|\mathcal{A}_{i}|].
\end{gather}
As $G^{\bm{Z}} = \{\Omega_{\beta}\}_{\beta=1}^{B}$, we can omit $(\{\Omega_{\gamma,i}\}_{\gamma=1}^{r+m} \setminus \{\Omega_{\beta,i}\}_{\beta=1}^{B})$ within Equation \ref{eq:adda} and have:
\begin{gather}
    \mathcal{L}_{o} = -\frac{1}{NK}\sum_{i=1}^{N} \sum_{k=1}^{K} \sum_{d=1}^{D} Y_{i}^{[k,d]} log \bigg( \text{norm}\Big(\sum_{a=1}^{|\mathcal{A}_i|} \Big(M_{i} \circ \bar{\delta}\big(\psi\big(\{\Omega_{\beta,i}\}_{\beta=1}^{B} \big)\big)\Big)^{[a]} +\Xi_{i}\Big)^{[d]}\bigg).
\end{gather}
As Assumption \ref{asp:perp} holds, the hypothesis we propose asserts that in certain context $S$ we have $\{\Omega_{\beta}\}_{\beta=1}^{B} \perp \!\!\!\perp Y^{[k,d]}$, and $\bar{\delta}(X)^{[d]} \neq 0$, we can generate a new $\psi'(\cdot)$ through only alter the parameters of $\psi(\cdot)$ and have:
\begin{gather}
     -\sum_{u=1}^{U} \sum_{k=1}^{K} \sum_{d=1}^{D} Y_{u}^{[k,d]} log \bigg( \text{norm}\Big(\sum_{a=1}^{|\mathcal{A}_u|} \Big(M_{u} \circ \bar{\delta}\big(\psi'\big(\{\Omega_{\beta,u}\}_{\beta=1}^{B} \big)\big)\Big)^{[a]} +\Xi_{u}\Big)^{[d]}\bigg)    \nonumber\\
     < -\sum_{u=1}^{U} \sum_{k=1}^{K} \sum_{d=1}^{D} Y_{u}^{[k,d]} log \bigg( \text{norm}\Big(\sum_{a=1}^{|\mathcal{A}_u|} \Big(M_{u} \circ \bar{\delta}\big(\psi\big(\{\Omega_{\beta,u}\}_{\beta=1}^{B} \big)\big)\Big)^{[a]} +\Xi_{u}\Big)^{[d]}\bigg), 
\end{gather}
by only letting $\psi'(\cdot)$ change the output for $\bar{\delta}\big(\psi'(\{\Omega_{\beta,u}\}_{\beta=1}^{B})^{[a]}\big)^{[d]}$ from $\bar{\delta}(X)^{[d]}$ to $0$, where $\{G_{u}\}_{u=1}^{U}$ denote the set of graph samples that have $\{\Omega_{\beta}\}_{\beta=1}^{B} \perp \!\!\!\perp Y^{[k,d]}$ holds. The reason behind this is $M_{i}$ ensures that when $\bar{\delta}(\bm{Z})^{[d]}=0$, it does not impact the output, and at this point, the output of the other parts of the model has already reached its optimum. Therefore, when $\bar{\delta}(\bm{Z})^{[d]}=0$, the output remains optimal. However, if $\bar{\delta}(\bm{Z})^{[d]} \neq 0$, then $\bar{\delta}(\bm{Z})^{[d]}$ will influence the output, and the impact from irrelevant information will cause the output to deviate from the optimal value.

Subsequently, we have:
\begin{align}
        \mathcal{L}_{o} &= -\frac{1}{NK}\sum_{i=1}^{N} \sum_{k=1}^{K} \sum_{d=1}^{D} Y_{i}^{[k,d]} log 
        \bigg( \text{norm}\Big(\sum_{l=1}^{|\mathcal{V}_i|} \Big(M_{i} 
        \circ \bar{\delta}\big(\psi(\{\Omega_{\gamma,i}\}_{\gamma=1}^{r+m})\big)\Big)^{[l]}\Big)^{[d]}\bigg) \nonumber\\
        &= -\frac{1}{NK}\sum_{w=1}^{W} \sum_{k=1}^{K} \sum_{d=1}^{D} Y_{i}^{[k,d]} log \bigg( \text{norm}\Big(\sum_{l=1}^{|\mathcal{V}_w|} \Big(M_{w} \circ \bar{\delta}\big(\psi(\{\Omega_{\gamma,w}\}_{\gamma=1}^{r+m})\big)\Big)^{[l]}\Big)^{[d]}\bigg)  \nonumber\\
        & + \sum_{u=1}^{U} \sum_{k=1}^{K} \sum_{d=1}^{D} Y_{u}^{[k,d]} log \bigg( \text{norm}\Big(\sum_{l=1}^{|\mathcal{V}_u|} \Big(M_{u} \circ \bar{\delta}\big(\psi(\{\Omega_{\gamma,u}\}_{\gamma=1}^{r+m})\big)\Big)^{[l]}\Big)^{[d]}\bigg)  \nonumber\\
        &= -\frac{1}{NK}\sum_{w=1}^{W} \sum_{k=1}^{K} \sum_{d=1}^{D} Y_{w}^{[k,d]}log \Big( \frac{1}{|\mathcal{V}_{w}|}\sum_{l=1}^{|\mathcal{V}_w|} \bar{\delta}\big( \psi(\{\Omega_{\gamma,w}\}_{\gamma=1}^{r+m} )^{[l]} \big)^{[d]}\Big) + \sum_{u=1}^{U} \sum_{k=1}^{K} \sum_{d=1}^{D} Y_{u}^{[k,d]} \nonumber\\
        &log\Big( \frac{1}{A_u}\sum_{a=1}^{A_u} \bar{\delta}\big(\psi(\{\Omega_{\beta,u}\}_{\beta=1}^{B})^{[a]}\big)^{[d]}+ \Xi_{u} \Big) \nonumber\\
        &> -\frac{1}{NK}\sum_{w=1}^{W} \sum_{k=1}^{K} \sum_{d=1}^{D} Y_{w}^{[k,d]}log \Big( \frac{1}{|\mathcal{V}_{w}|}\sum_{l=1}^{|\mathcal{V}_w|} \bar{\delta}\big( \psi'(\{\Omega_{\gamma,w}\}_{\gamma=1}^{r+m} )^{[l]} \big)^{[d]}\Big) + \sum_{u=1}^{U} \sum_{k=1}^{K} \sum_{d=1}^{D} Y_{u}^{[k,d]} \nonumber\\
        &log\Big( \frac{1}{A_u}\sum_{a=1}^{A_u} \bar{\delta}\big(\psi'(\{\Omega_{\beta,u}\}_{\beta=1}^{B})^{[a]}\big)^{[d]}+ \Xi_{u} \Big), 
\end{align}
where $\{G_{w}\}_{w=1}^{W} = \{G_{i}\}_{i=1}^{N} \setminus \{G_{u}\}_{u=1}^{U}$. Therefore $\mathcal{L}_{o}$ can be further minimized, which is contrary to the given conditions, the hypothesis we make does not hold, therefore if the condition $\delta(\bm{Z})^{[d]} \neq 0$ is satisfied, then $G^{\bm{Z}}$ is related to $Y^{[k,d]}$ in all contexts, therefore related to $Y$ in all contexts. 

The above conclusion demonstrates that under the given conditions, $G^{\bm{Z}}$ satisfied the first condition of Definition \ref{def:poc}, now we only need to show $G^{\bm{Z}}$ also satisfied the second principle of Definition \ref{def:poc} to prove the theorem. From the SCM in Figure \ref{fig:SCM}, it is obvious that there is always a causal root from $\{\Omega_{\beta}\}_{\beta=1}^{B}$ towards $Y^{*}$ or $\widetilde{Y}$ as we have proved that under the given condition, $G^{\bm{Z}}$, i.e. $\{\Omega_{\beta}\}_{\beta=1}^{B}$, is related to $Y^{[:,d]}$ in all contexts. Next, we can choose any $\Omega$ within $\{\Omega_{1}, ..., \Omega_{r}\}$ other than those within $\{\Omega_{\beta}\}_{\beta=1}^{B}$, and given all $\{\bar{\text{PA}}(\Omega_{i})\}_{i=1}^{r+m}$, the chosen $\Omega$ still related to $Y$, therefore the second condition of Definition \ref{def:poc} holds, the theorem is proved.

\subsection{A.3. Proof of Theorem \ref{thm:leq}}
\label{prf:2}
According to Theorem \ref{thm:leq}, $\mathcal{L}_{1}$ can be formulated as follows:
\begin{align}
    \mathcal{L}_{1} &= \sum_{i=1}^{N} \sum_{k=1}^{K} \mathcal{H}\Big( \text{epow}(f_{\theta},\lambda), Y_{i}^{[k]}\Big) \nonumber\\
    &= \sum_{i=1}^{N} \mathcal{H}\Big(\text{epow}(f_{\theta},\lambda), Y_{i}^{*} \Big)  +  \sum_{i=1}^{N} \mathcal{H}\Big(\text{epow}(f_{\theta},\lambda), \widetilde{Y}_{i} \Big) 
    \nonumber\\
    &= -\sum_{i=1}^{N}\sum_{d=1}^{D} Y_{i}^{*[d]}log\big(e^{(f(G_{i}^{*})^{[d]})^{\lambda}}\big) 
    - \sum_{i=1}^{N}\sum_{d=1}^{D} \widetilde{Y}_{i}^{[d]}log\big(e^{(f(G_{i}^{*})^{[d]})^{\lambda}}\big) 
    \nonumber\\
    &= -\sum_{i=1}^{N}\sum_{d=1}^{D} Y_{i}^{*[d]}\big(f(G_{i}^{*})^{[d]}\big)^{\lambda}  
     -\sum_{i=1}^{N}\sum_{d=1}^{D} \widetilde{Y}_{i}^{[d]}\big(f(G_{i}^{*})^{[d]}\big)^{\lambda}.    
    \label{eq:prffy}
\end{align}
As $Y_{i}^{*}$ and $\widetilde{Y}_{i}$ are one-hot vectors, Equation \ref{eq:prffy} can be formulated as:
\begin{gather}
\mathcal{L}_{1} = -\sum_{i=0}^{N} \big(\Psi_{i,d^{*}}\big)^{\lambda}
    -\sum_{i=0}^{N} \big(\Psi_{i,\widetilde{d}})^{\lambda},   
\end{gather}
where $\Psi_{i}$ denote the output of $f(\cdot)$ with $G_{i}^{*}$ as the input. $\Psi_{i,d^{*}}$ denote $d^{*}$-th dimension of $\Psi_{i}$, $Y_{i,d^{*}}^{*} = 1$. Likewise, $\Psi_{i,\widetilde{d}}$ denote $\widetilde{d}$-th dimension of $\Psi_{i}$, $\widetilde{Y}_{i,\widetilde{d}} = 1$. Meanwhile, We have:
\begin{gather}
\sum_{d=1}^{D}\Psi_{i,d} = 1, \forall i \in \{1,2,...,N\}.
\label{eq:sum}
\end{gather}
According to the definition of $\mathcal{L}_{1}$, $\lambda > 1$, then $\big(X_{i,d^{*}}\big)^{\lambda}$ is a strictly concave function of $X_{i,d^{*}}$. Therefore, we assert that if and only if $\Psi_{i,d^{*}}=1, \forall i \in \{1,2,...,N\}$, $\mathcal{L}_{1}$ reaches minimal. We will substantiate this claim by analyzing the problem within different cases.

As Equation \ref{eq:sum} holds, if $d^{*} \neq \widetilde{d}$, we have:
\begin{gather}
\Psi_{i,1} + ... + \Psi_{i,d^{*}} + ... + \Psi_{i,\widetilde{d}} + ... + \Psi_{i,D} = 1,
\label{eq:a1}
\end{gather}
where $\Psi_{i,\widetilde{d}}$ is not necessarily positioned after $\Psi_{i,d^{*}}$, in some cases $d^{*} = \widetilde{d}$. Therefore, we can divide the studied problem into three cases: 1) $\Psi_{i,d^{*}}=1$ or $\Psi_{i,\widetilde{d}}=1$ or $\Psi_{i,d^{*}}=\Psi_{i,\widetilde{d}}=1$ ; 2) $\Psi_{i,d}=1, d \neq d^{*} \ \text{and} \ d \neq \widetilde{d} $ ; 3) $\Psi_{i,d}<1,\forall i \in \{1,2,...,D\}$. 

For Case 1, according to the theorem, condition $P(\widetilde{Y}=\widetilde{Y}_{i}|G^{*}=G_{i}^{*})\neq 1, \forall i \in \{1,2,...,N\}$ holds. Consequently, it is not feasible to predict $\widetilde{Y}$ accurately solely based on $G^{*}$. Therefore, the only condition under which $\mathcal{L}_{1}$ achieves its minimum for Case 1 is when $\Psi_{i,d^{*}}=1$ for every $i \in \{1,2,...,N\}$.

For Case 2, we have:
\begin{align}
\mathcal{L}_{1} &= -\sum_{i=0}^{N} \big(\Psi_{i,d^{*}}\big)^{\lambda}
    -\sum_{i=0}^{N} \big(\Psi_{i,\widetilde{d}})^{\lambda} = -\sum_{i=0}^{N} \big(0\big)^{\lambda}
    -\sum_{i=0}^{N} \big(0)^{\lambda} = 0,   
\end{align}
This value is evidently greater than the minimum of $\mathcal{L}_{1}$ in Case 1. Consequently, in Case 2, $\mathcal{L}_{1}$ cannot achieve its minimum value.

For Case 3, based on Equation 15, it can be seen that variables other than $\Psi_{i,d^{*}}$ and $\Psi_{i,\widetilde{d}}$ do not contribute to the value of $\mathcal{L}_{1}$. Furthermore, an increase in the values of $\Psi_{i,d^{*}}$ and $\Psi_{i,\widetilde{d}}$ leads to a decrease in $\mathcal{L}_{1}$. Therefore, based on Equation \ref{eq:a1}, when $\mathcal{L}_{1}$ reaches its minimal possible value in Case 3, if $d^{*} \neq \widetilde{d}$,  $\Psi_{i,d^{*}} + \Psi_{i,\widetilde{d}} = 1$. However, as Case 3 constraint that $\Psi_{i,d}<1,\forall i \in \{1,2,...,D\}$, we can conclude that:
\begin{align}
    \mathcal{L}^{\text{Case 3}}_{1} &= -\sum_{i=0}^{N} \big(\Psi^{\text{Case 3}}_{i,d^{*}}\big)^{\lambda}
    -\sum_{i=0}^{N} \big(\Psi^{\text{Case 3}}_{i,\widetilde{d}})^{\lambda} \nonumber\\
    &= -\sum_{i=0}^{N} \Big(\big(\Psi^{\text{Case 3}}_{i,d^{*}}\big)^{\lambda} + \big(\Psi^{\text{Case 3}}_{i,\widetilde{d}})^{\lambda}\Big) \nonumber\\
    &< -\sum_{i=0}^{N} \big(\Psi^{\text{Case 1}}_{i,d^{*}}\big)^{\lambda} = \mathcal{L}^{\text{Case 1}}_{1}, 
    \label{eq:leqpsi}
\end{align}
where $\mathcal{L}^{\text{Case 3}}_{1}$ and $\mathcal{L}^{\text{Case 1}}$ denotes the minimal value of $\mathcal{L}_{1}$ within Case 3 and Case 1 when $d^{*} \neq \widetilde{d}$, $\Psi^{\text{Case 3}}_{i,d^{*}}$ and $\Psi^{\text{Case 3}}_{i,\widetilde{d}}$ denote the corresponding output values in Case 3, $\Psi^{\text{Case 1}}_{i,d^{*}}$ is the corresponding output value in Case 1. Equation \ref{eq:leqpsi} holds because that as $\Psi^{\text{Case 1}}_{i,d^{*}} = 1 = \Psi^{\text{Case 3}}_{i,d^{*}} + \Psi^{\text{Case 3}}_{i,\widetilde{d}}$, then with $0 < \tau < 1$ we can denote that $\tau\Psi^{\text{Case 1}}_{i,d^{*}} = \Psi^{\text{Case 3}}_{i,d^{*}} $ and $(1-\tau)\Psi^{\text{Case 1}}_{i,d^{*}} = \Psi^{\text{Case 3}}_{i,\widetilde{d}}$, according to the convexity inequality, we have:
\begin{gather}
    \tau\big(\Psi^{\text{Case 1}}_{i,d^{*}}\big)^{\lambda} + (1-\tau)\big(0)^{\lambda} > 
    \big(\tau\Psi^{\text{Case 1}}_{i,d^{*}} + (1-\tau)\big(0)\big)^{\lambda} =
    \big(\Psi^{\text{Case 3}}_{i,d^{*}} \big)^{\lambda}
\end{gather}
and:
\begin{gather}
    (1-\tau)\big(\Psi^{\text{Case 1}}_{i,d^{*}}\big)^{\lambda} + (\tau)\big(0)^{\lambda} > 
    \big((1-\tau)\Psi^{\text{Case 1}}_{i,d^{*}} + (\tau)\big(0)\big)^{\lambda} =
    \big(\Psi^{\text{Case 3}}_{i,\widetilde{d}} \big)^{\lambda}.
\end{gather}
Then we can conclude that:
\begin{gather}
    \Psi^{\text{Case 1}}_{i,d^{*}} > \big(\Psi^{\text{Case 3}}_{i,d^{*}}\big)^{\lambda} + \big(\Psi^{\text{Case 3}}_{i,\widetilde{d}})^{\lambda}.
\end{gather}
For $d^{*} = \widetilde{d}$, as $\Psi_{i,d}<1,\forall i \in \{1,2,...,D\}$, then $\mathcal{L}_{1}$ in Case 3 also can't reach a value that smaller than that within Case 1. So far, we have proved our claim that if and only if $\Psi_{i,d^{*}}=1, \forall i \in \{1,2,...,N\}$, $\mathcal{L}_{1}$ reaches minimal, which means under such condition $f(\cdot)$ can accurately predict $Y^{*}$ based on $G^{*}$. Therefore, $\mathcal{L}_{2}$ also reaches its minimal value, and the theorem is proved.

\subsection{A.4. Proof of Corollary \ref{cly:tt}}
\label{prf:3}
Based on the proof of Theorem \ref{thm:lag} We can significantly increase the value of $\lambda$, and as a result, we obtain:
\begin{gather}
    \lim_{{\lambda \to \inf}} (\Psi_{i,d})^{\lambda} = 0,
    0 \leq \Psi_{i,d} < 1.
\end{gather}
which means by setting $\lambda$ to a sufficiently large value, we make any $(\Psi_{i,d})^{\lambda}$ with $\Psi_{i,d} < 1 $ approaches zero. The analysis from the proof of Theorem \ref{thm:leq} indicates that at this point, the loss can only reach its minimum value when $\Psi_{i,d^{*}}=1$ for every $i\in \{1,2,...,N\}$, which is the same conclusion as Theorem \ref{thm:lag}. The corollary has thus been validated.

\label{apx:k9}
\begin{table*}[t]\scriptsize
        \caption{Summary of datasets.}
	\begin{center}
		\begin{tabular}{lcccccc}
			\hline\rule{0pt}{5pt}
			
			{Name}  &  Graphs\#  & Average Nodes\# & Average Edges\# & Classes\# & Task Type & Metric  \\ 	
			\hline\rule{-3pt}{10pt}
			\text{Graph-SST5(OOD)} & {10700} & {21.29} &{40.58} & {5} & {Classification} & {ACC}\\
			\text{Graph-SST5(ID)} & {11855} & {19.85} &{37.70} & {5} & {Classification} & {ACC}\\
			\text{Graph-Twitter(OOD)} & {6344} & {21.96} &{41.92} & {3} & {Classification} & {ACC}\\
                \text{Graph-Twitter(ID)} & {6940} & {21.10} &{40.20} & {3} & {Classification} & {ACC}\\
			\text{Graph-SST2} & {70042} & {10.20} &{18.40} & {2} & {Binary Classification} & {ACC}\\
			\text{COLLAB} & {5000} & {74.49} &{2457.78} & {3} & {Classification} & {ACC}\\
			\text{REDDIT-MULTI-5K} & {4999} & {508.52} &{594.87} & {5} & {Classification} & {ACC}\\

			\hline
		\end{tabular}
	\end{center}
        \label{table:dataset}
\end{table*}

\begin{table*}[t]\scriptsize
 	\caption{Summary of the backbones used in each dataset.}
	\begin{center}
		\begin{tabular}{lccccc}
			\hline\rule{0pt}{10pt}
			
			{Name}  &  Backbone\#  & Size of $g$ (GNN) & Size of $\delta$ (MLP) & Global Pool\\ 	
			\hline\rule{-3pt}{10pt}
			\text{Graph-SST5(OOD)} & {ARMA} & {[768,256,128]} & {[128,5]} & {global mean pool}\\
			\text{Graph-SST5(ID)} & {ARMA} & {[768,256,128]} & {[128,5]} & {global mean pool}\\
			\text{Graph-Twitter(OOD)} & {ARMA} & {[768,256,128]} & {[128,3]} & {global mean pool}\\
                \text{Graph-Twitter(ID)} & {ARMA} & {[768,256,128]} & {[128,3]} & {global mean pool}\\
			\text{Graph-SST2} & {ARMA} & {[768,256,128]} & {[128,2]} & {global mean pool}\\
			\text{COLLAB} & {ARMA} & {[768,256,128]} & {[128,3]} & {global mean pool}\\
			\text{REDDIT-MULTI-5K} & {ARMA} & {[32,256,128]} & {[128,5]} & {global mean pool}\\

			\hline
		\end{tabular}
	\end{center}
        \label{table:backbone}
\end{table*}

\begin{figure*}
	\centering
	\subfigure[Training: Positive sentiment]{
			\includegraphics[width=0.32\linewidth]{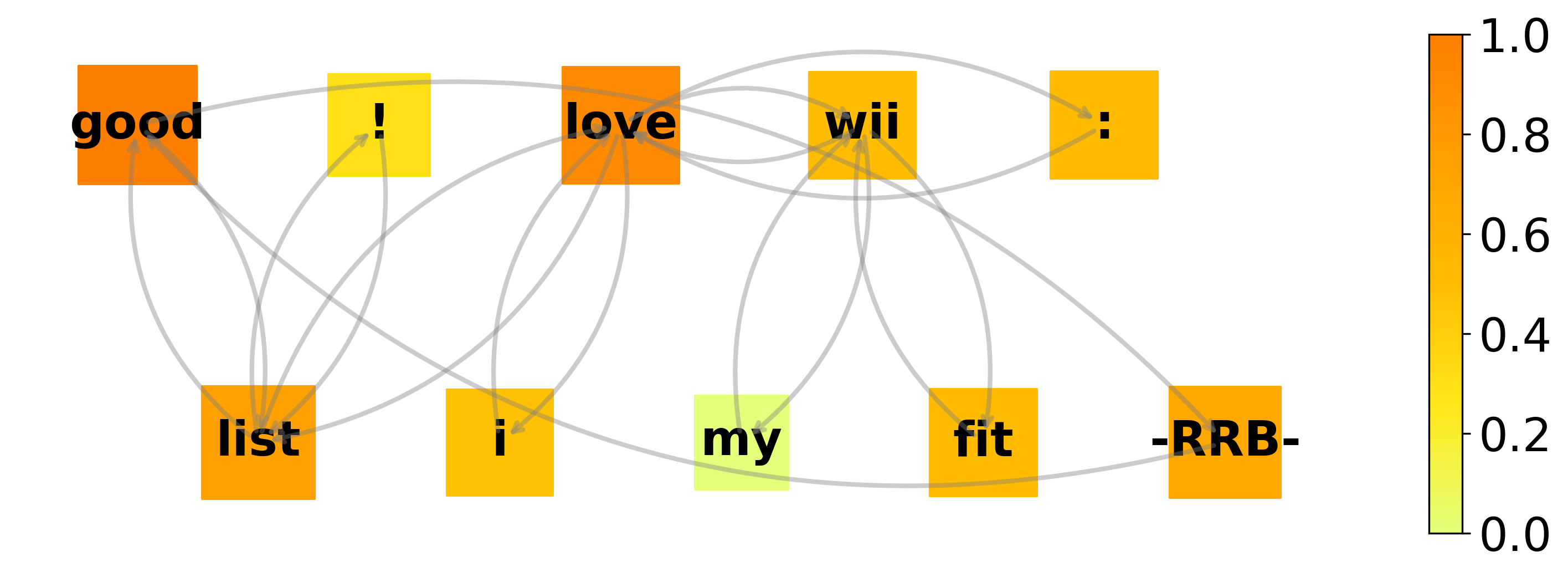}
                \includegraphics[width=0.32\linewidth]{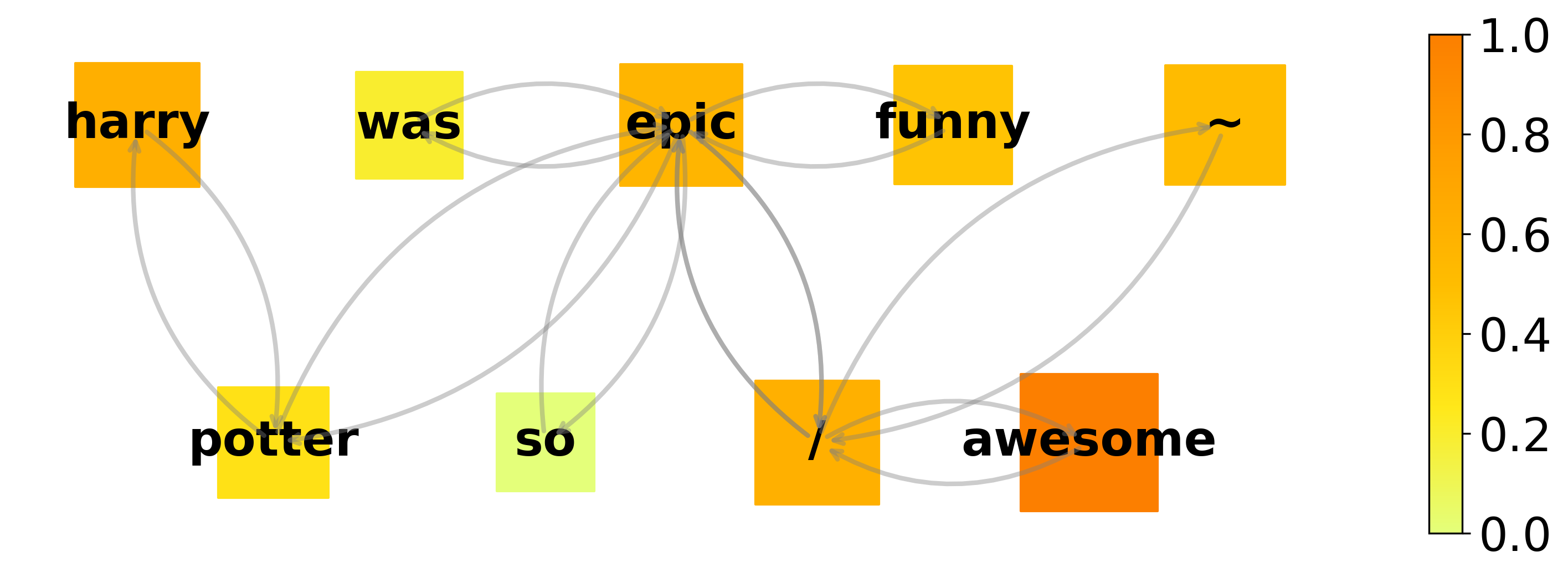}
                \includegraphics[width=0.32\linewidth]{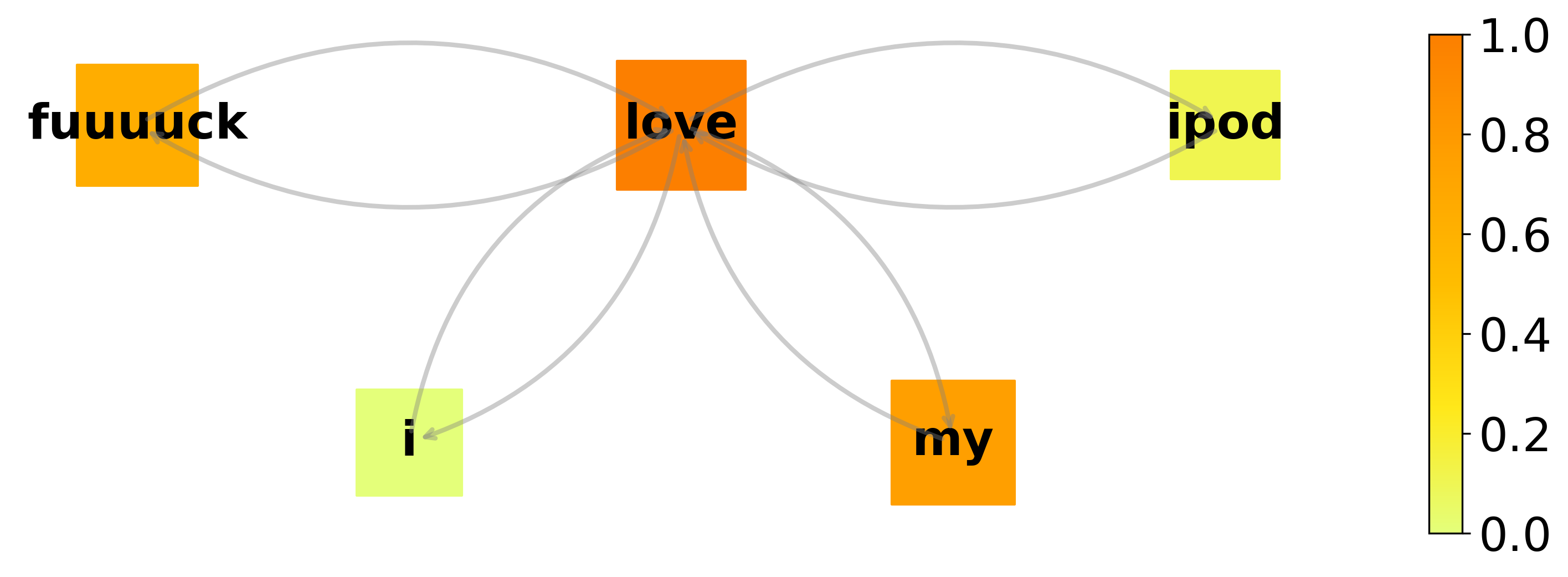}
	}\\
	\subfigure[Training: Neutral sentiment]{
			\includegraphics[width=0.32\linewidth]{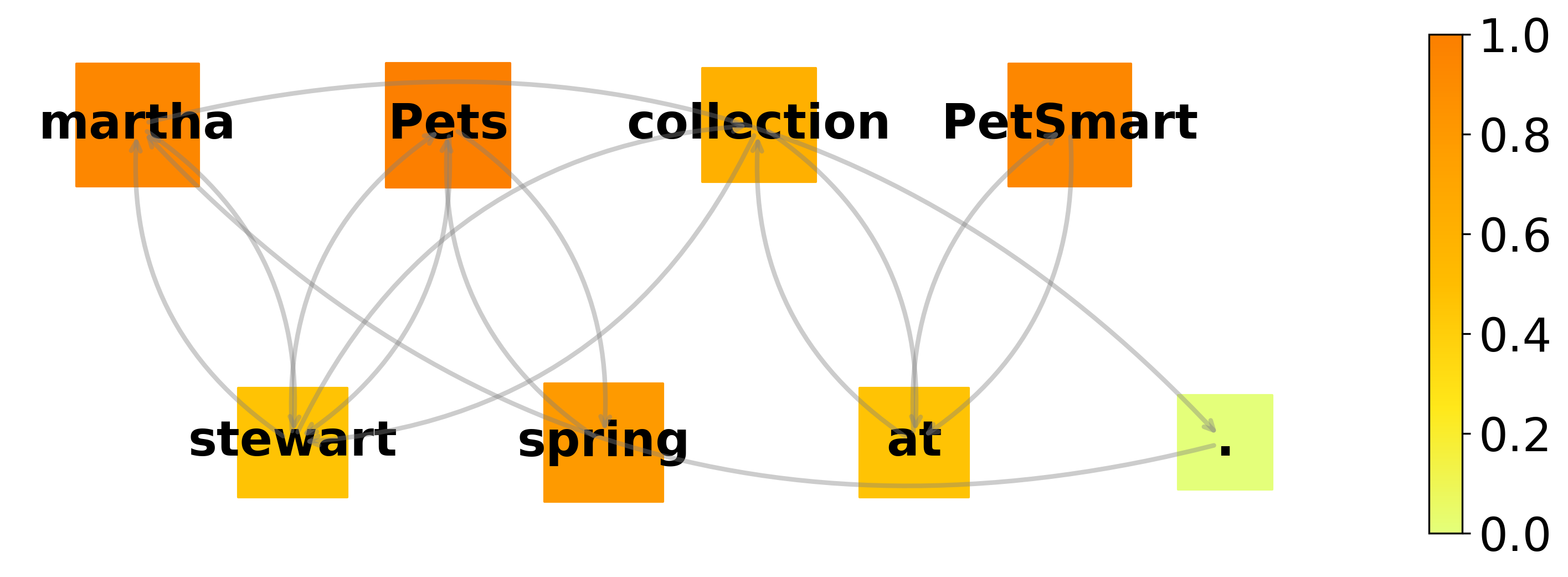}
			\includegraphics[width=0.32\linewidth]{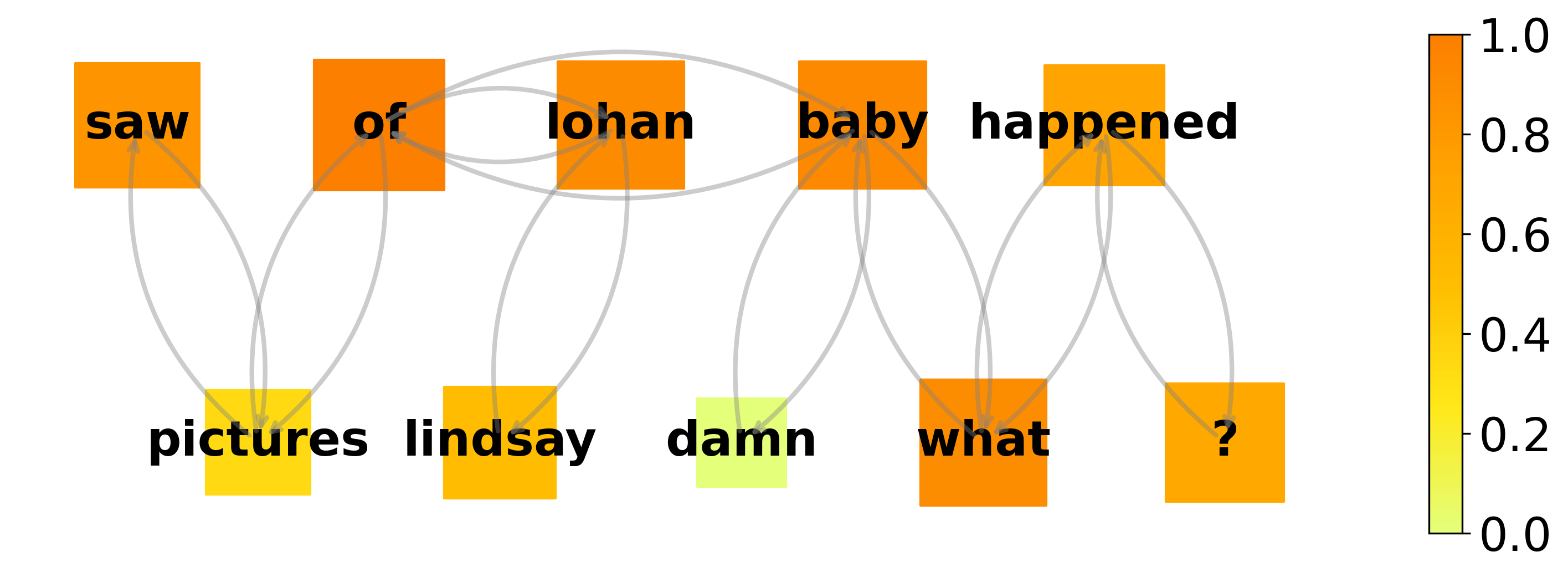}
			\includegraphics[width=0.32\linewidth]{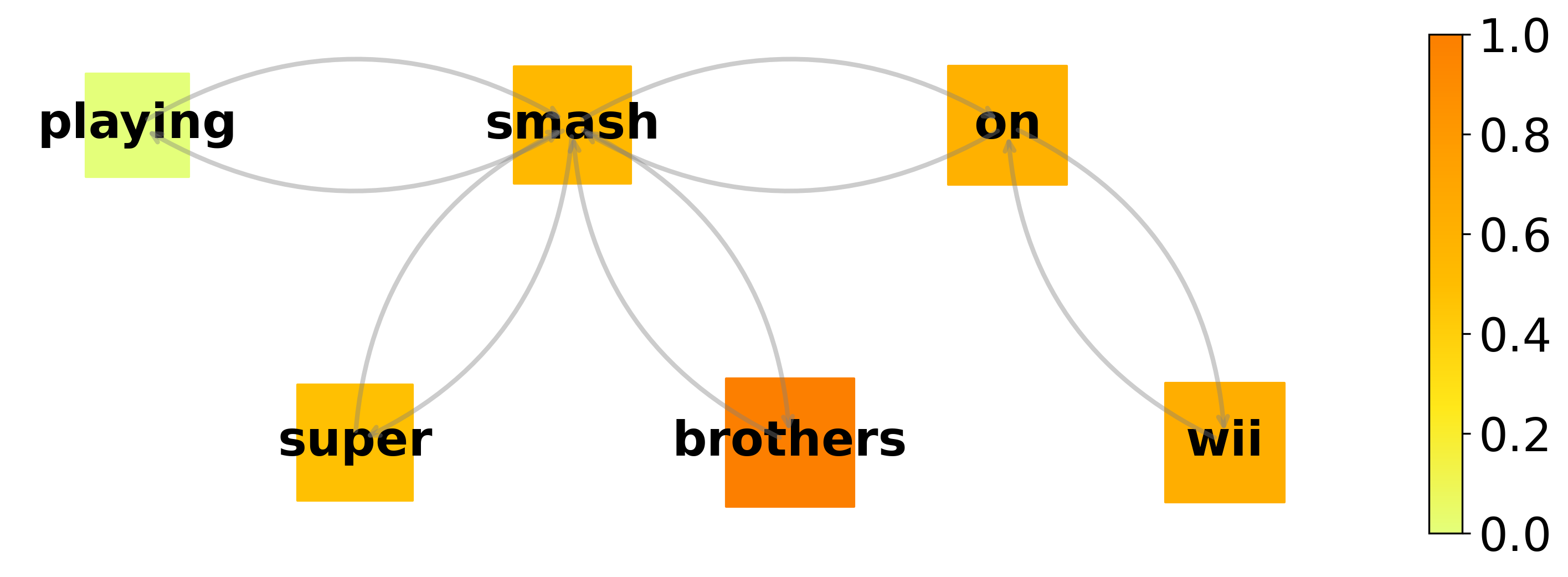}
	}\\
	\subfigure[Training: Negative sentiment]{
			\includegraphics[width=0.32\linewidth]{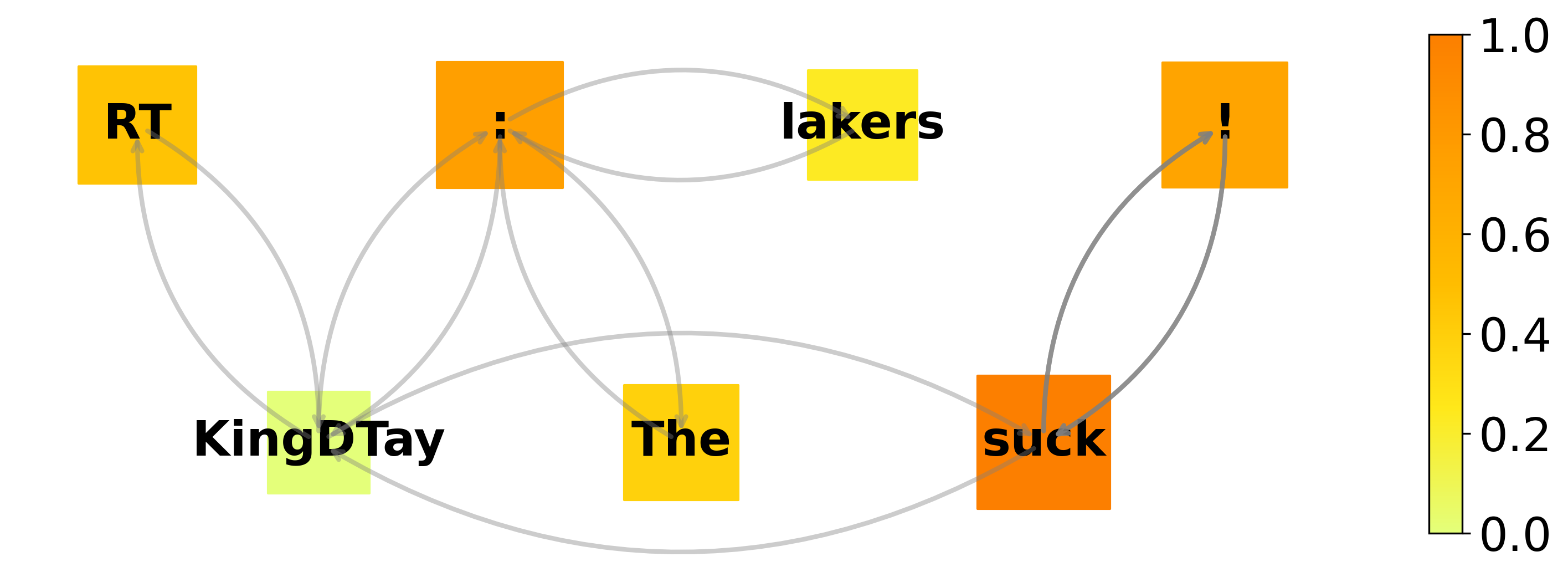}
			\includegraphics[width=0.32\linewidth]{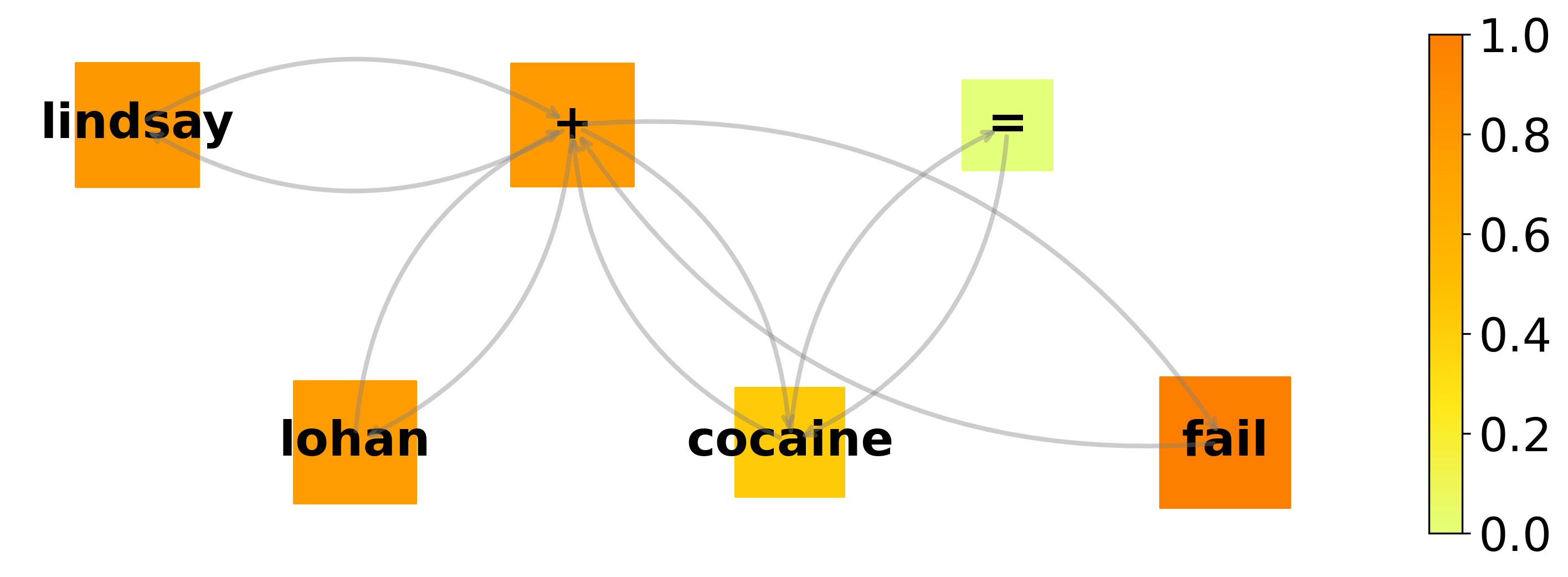}
			\includegraphics[width=0.32\linewidth]{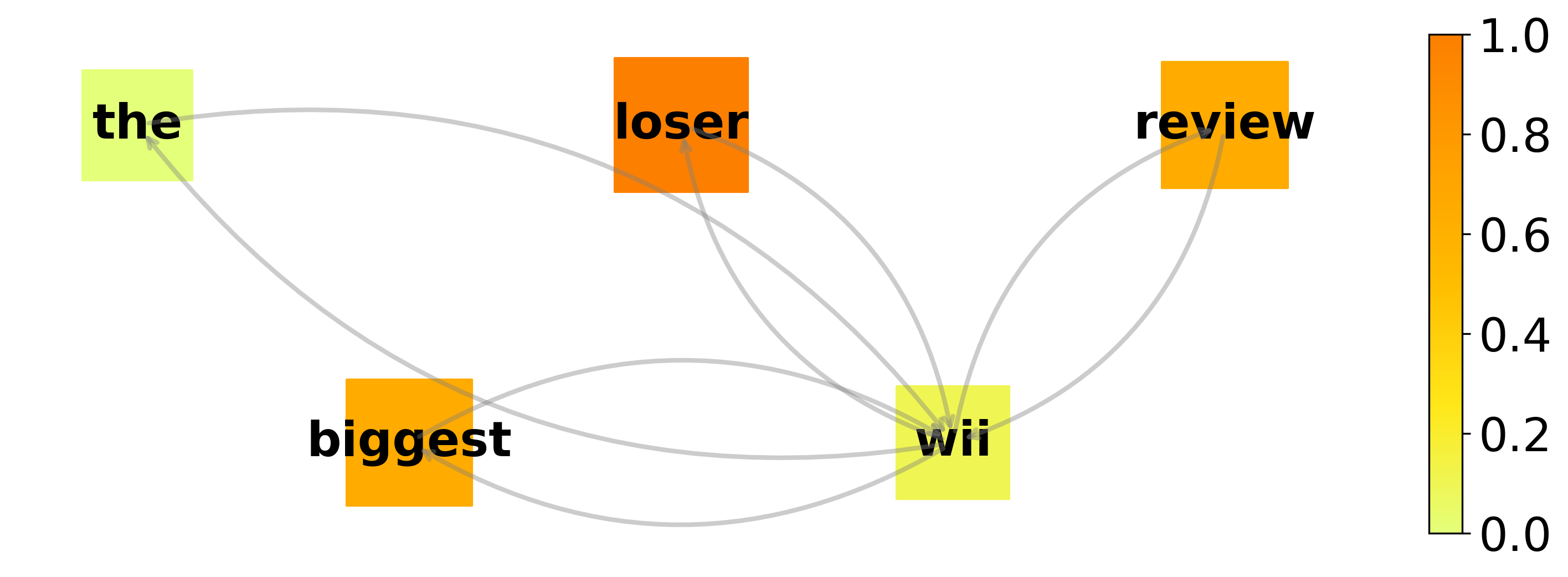}
	}\\
 
	\centering
	\caption{Visualization of node-level prediction vector on the training set of the Graph-Twitter dataset. Each graph represents a sentence.}
	\label{fig:sentiment2}
\end{figure*}

\begin{figure*}
	\centering
	\subfigure[Testing: Positive sentiment]{
			\includegraphics[width=0.32\linewidth]{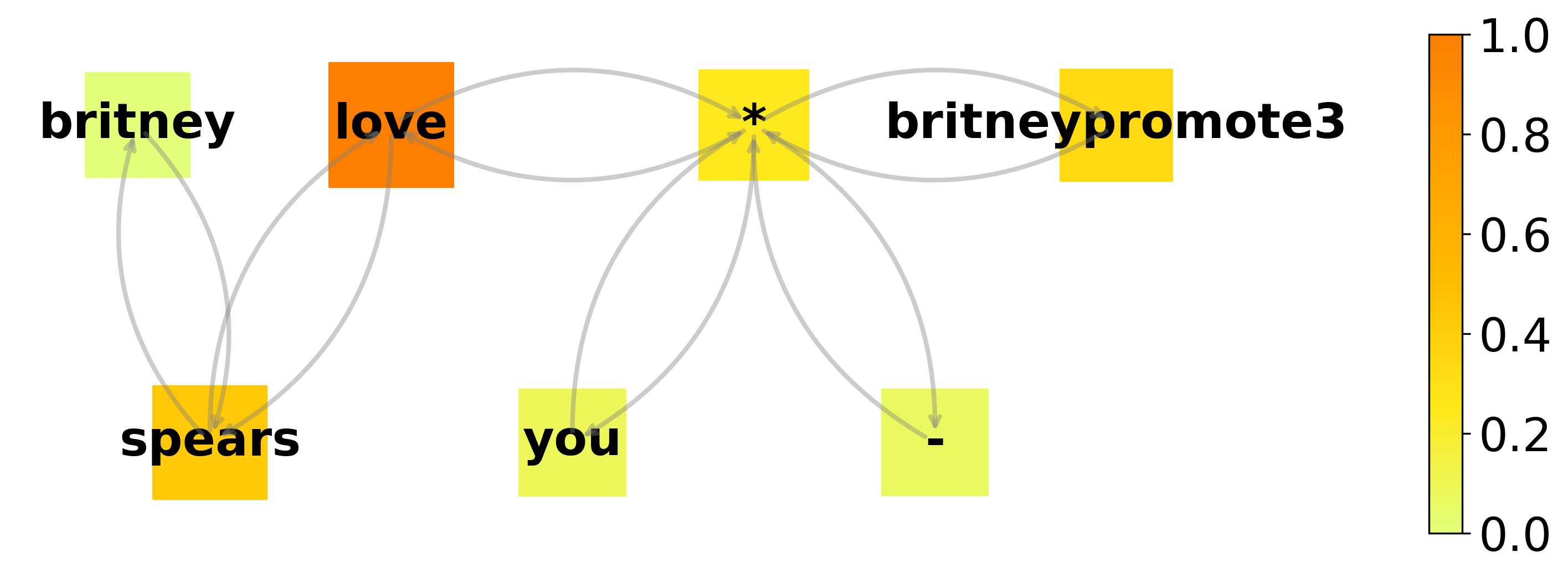}
                \includegraphics[width=0.32\linewidth]{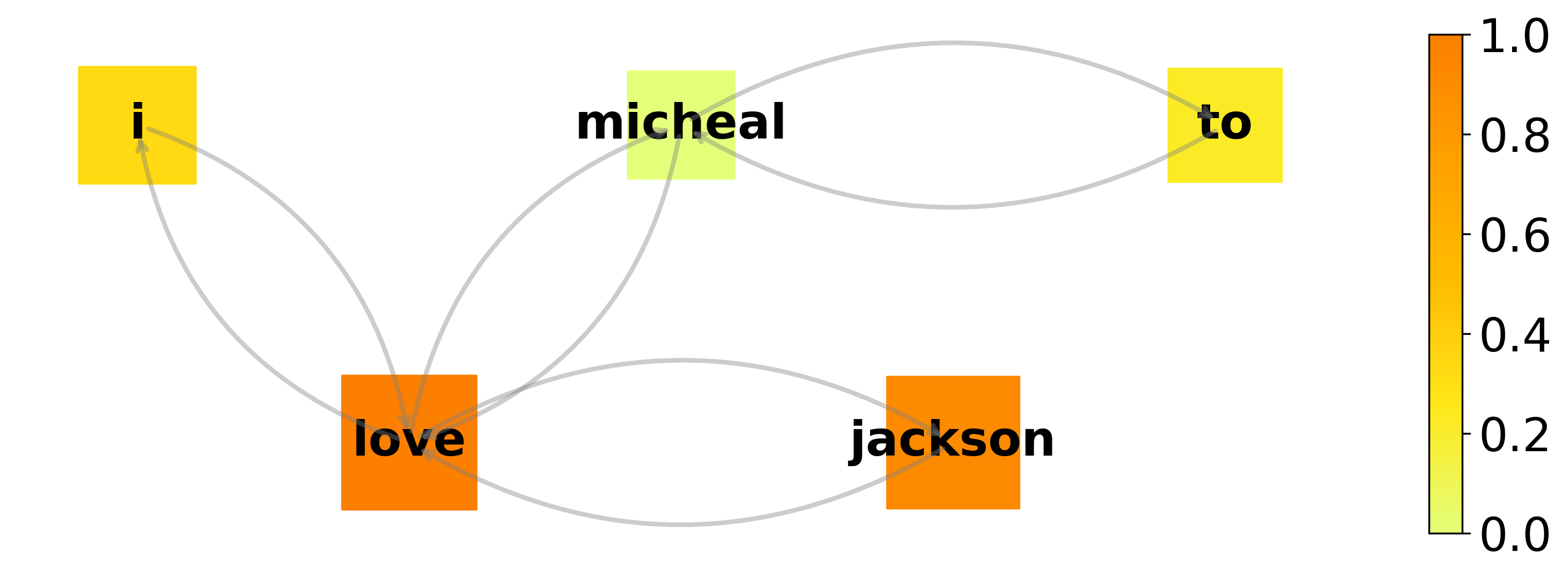}
                \includegraphics[width=0.32\linewidth]{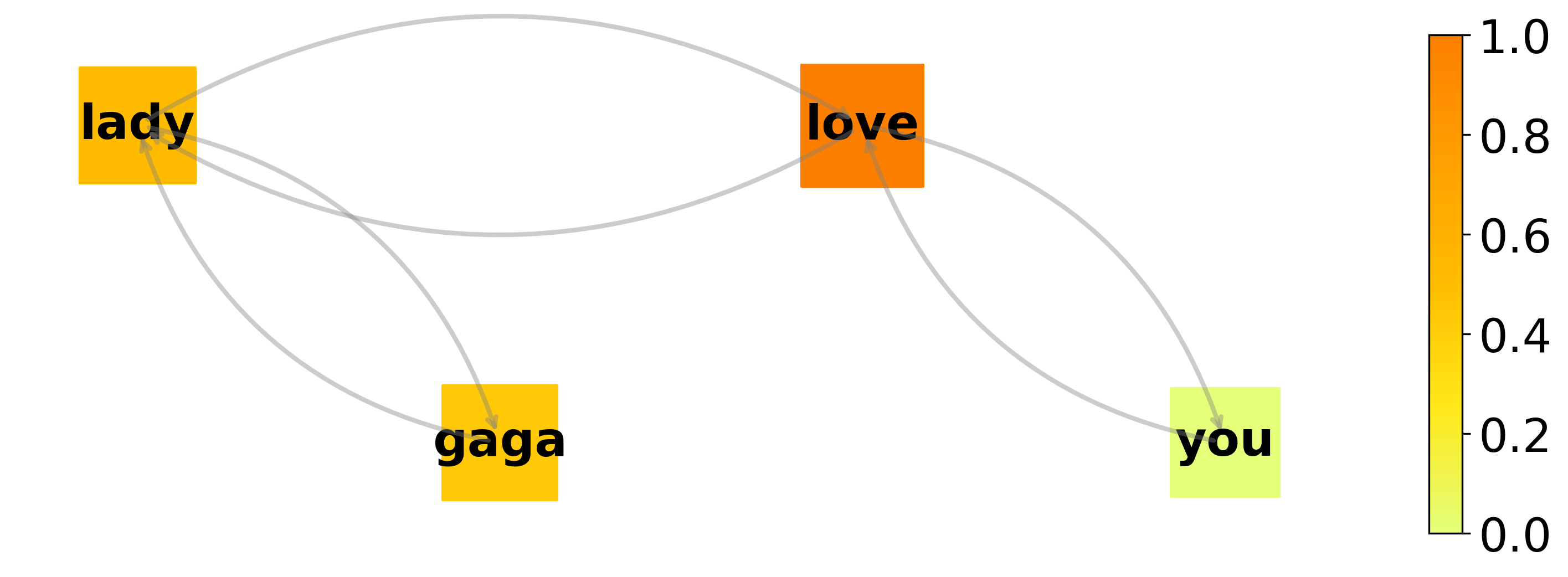}
	}\\
	\subfigure[Testing: Neutral sentiment]{
			\includegraphics[width=0.32\linewidth]{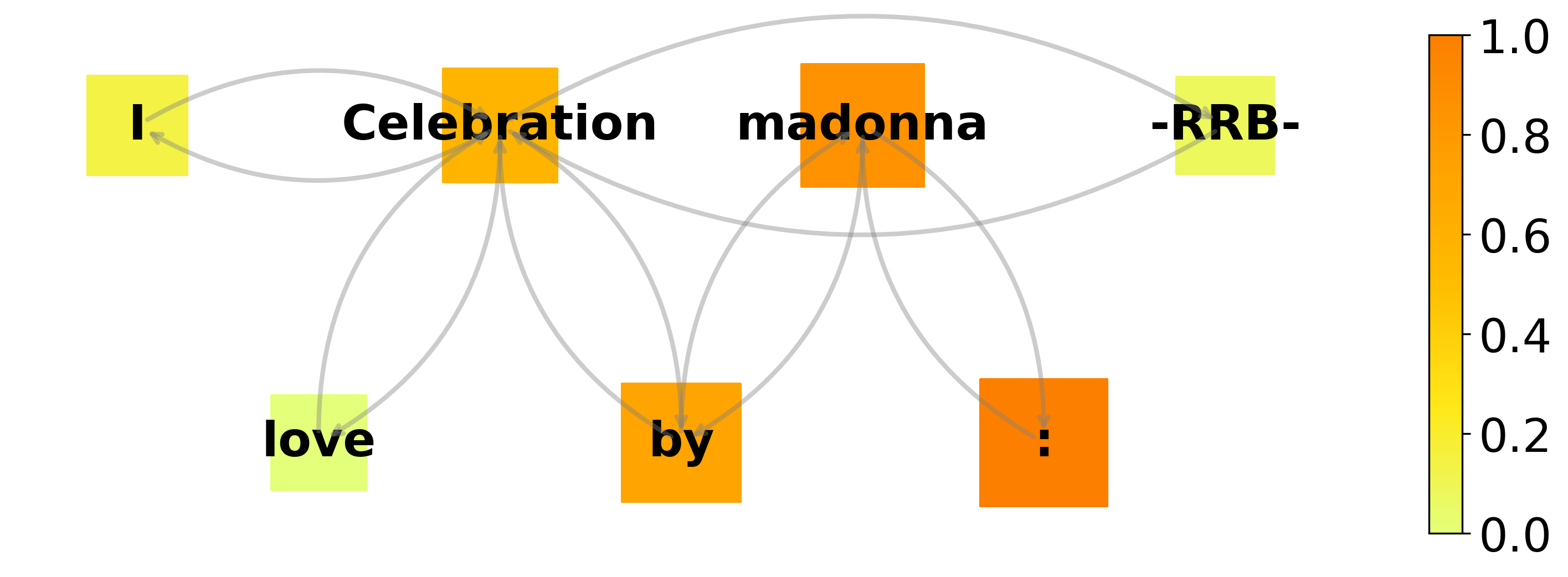}
			\includegraphics[width=0.32\linewidth]{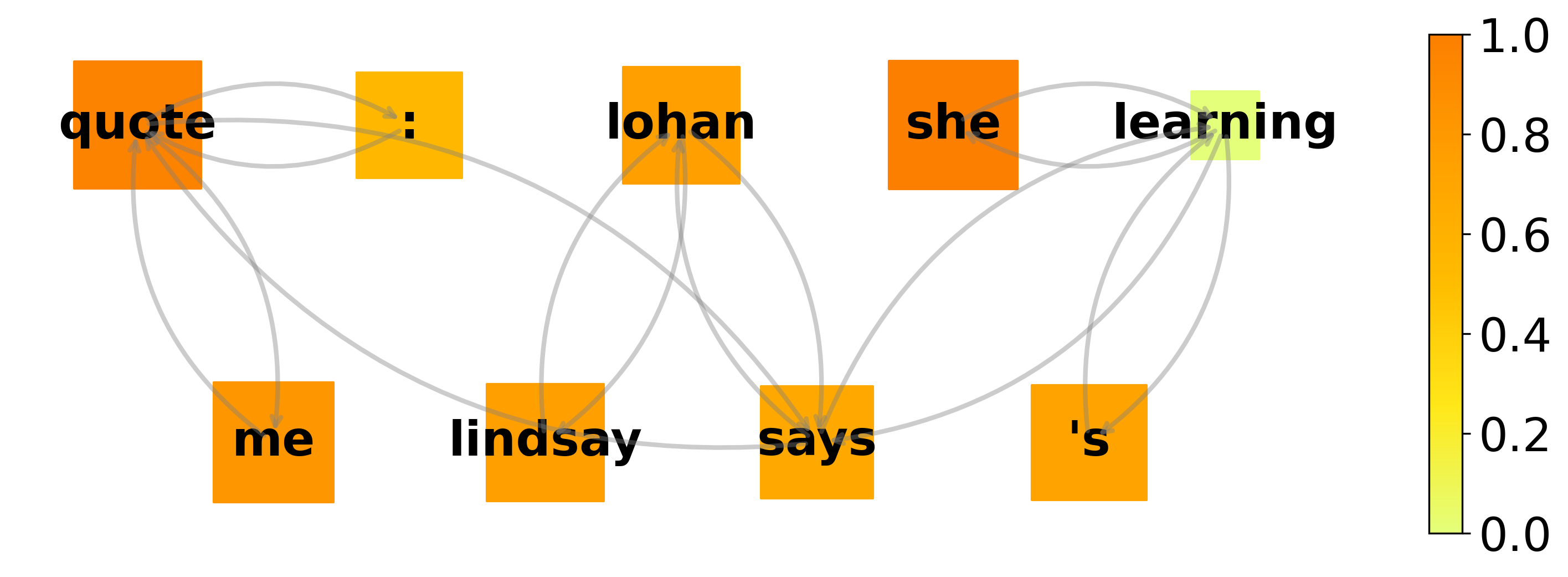}
			\includegraphics[width=0.32\linewidth]{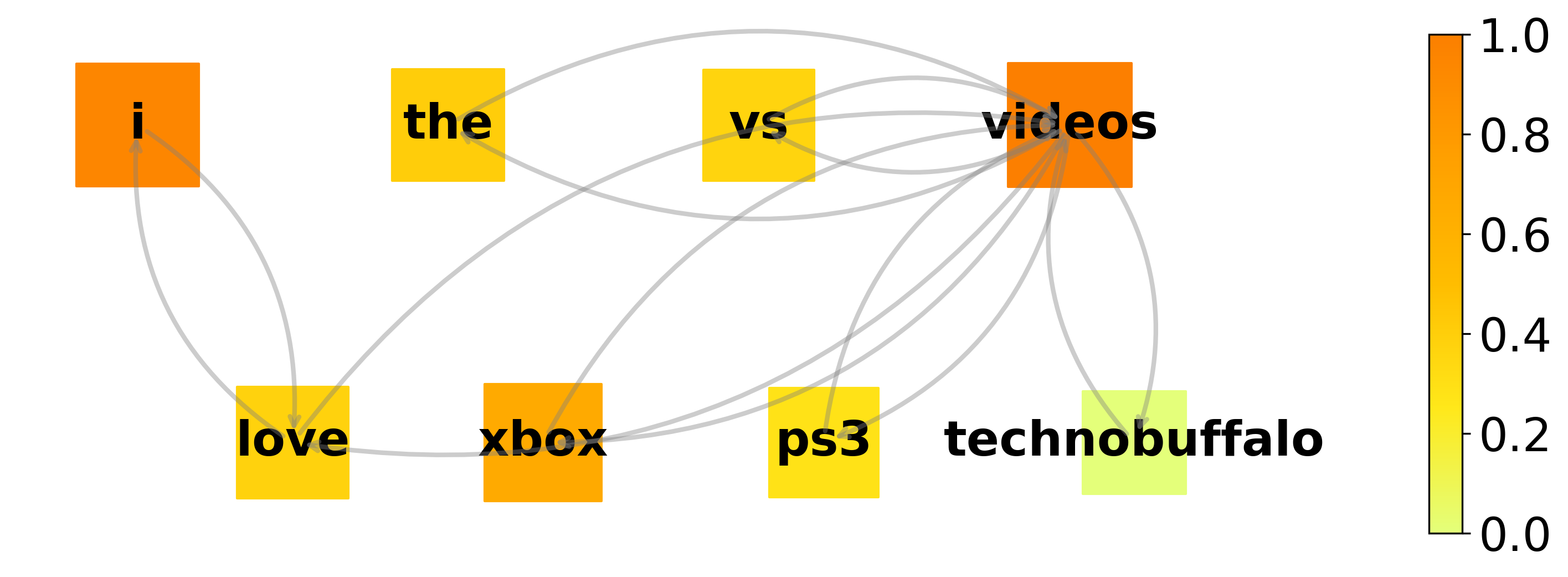}
	}\\
	\subfigure[Testing: Negative sentiment]{
			\includegraphics[width=0.32\linewidth]{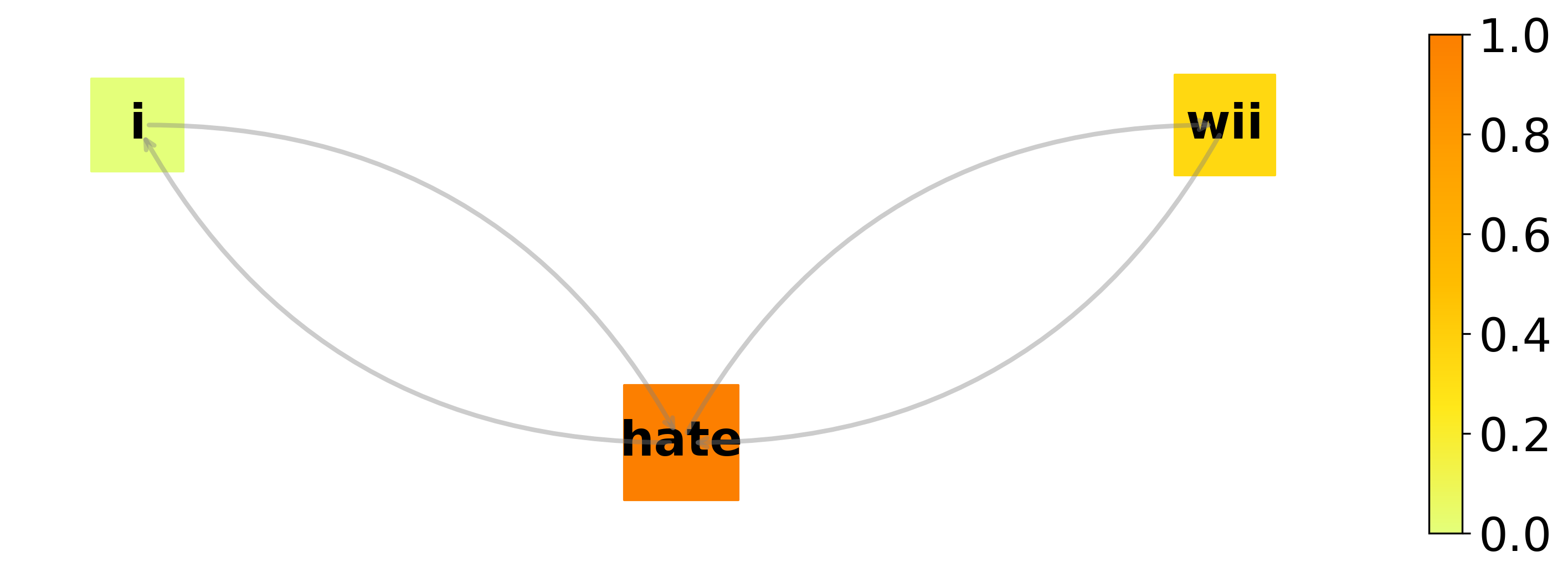}
			\includegraphics[width=0.32\linewidth]{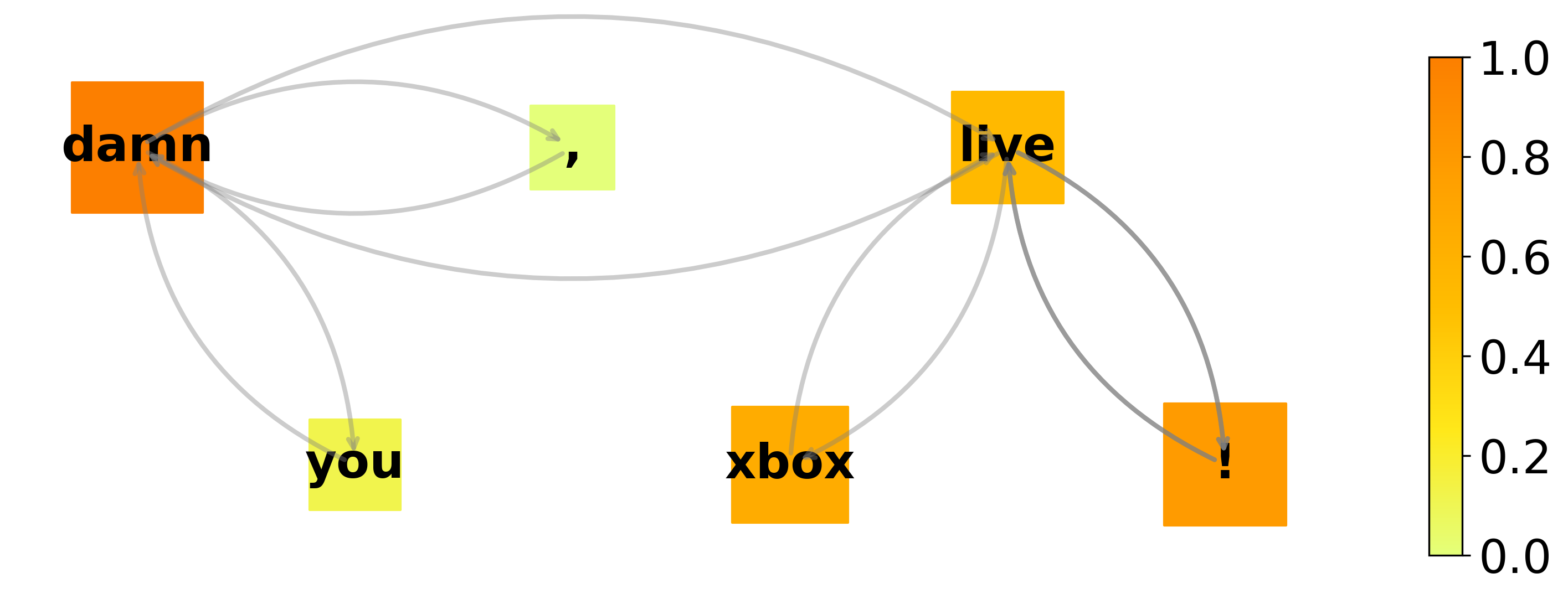}
			\includegraphics[width=0.32\linewidth]{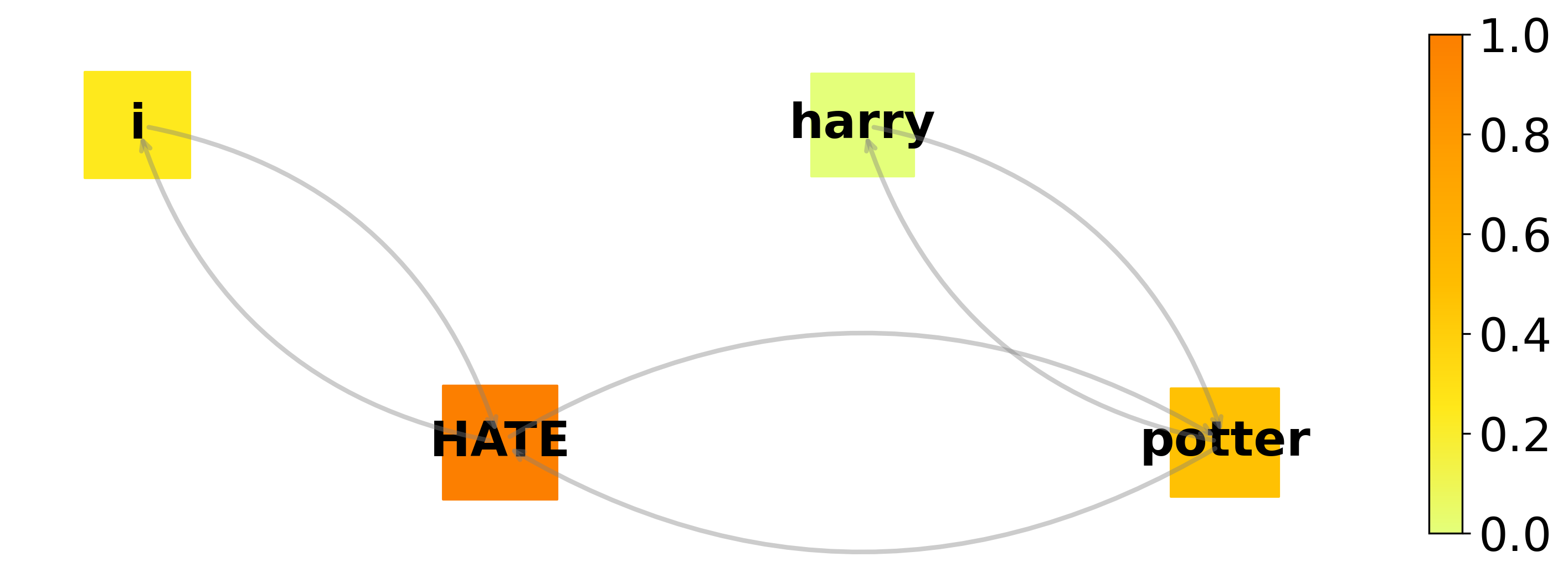}
	}\\
 
	\centering
	\caption{Visualization of node-level prediction vector on the testing set of the Graph-Twitter dataset. Each graph represents a sentence.}
	\label{fig:sentiment3}
\end{figure*}

\begin{figure}
	\centering
	\subfigure[]{
		\begin{minipage}[t]{0.35\linewidth}
			\centering
			\includegraphics[width=1\linewidth]{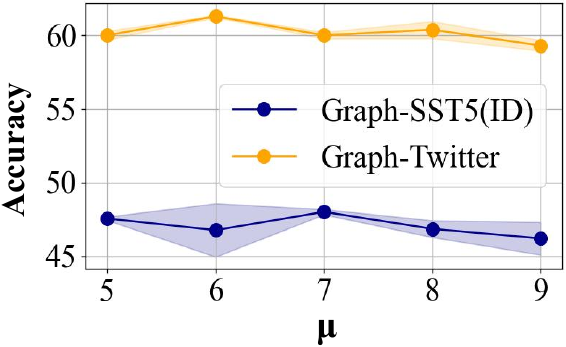}
       \vskip -0.15in
            \label{fig:parameter-a}
		\end{minipage}%
	}%
	\subfigure[]{
		\begin{minipage}[t]{0.35\linewidth}
			\centering
			\includegraphics[width=1\linewidth]{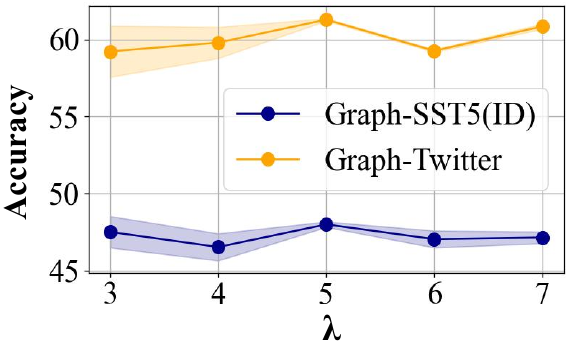}
       \vskip -0.15in
            \label{fig:parameter-b}
		\end{minipage}%
	}%
	\centering
    \vskip -0.1in
	\caption{Performance of GPCD with different $\mu$ and $\lambda$.}
	\label{fig:parameter}
     \vskip -0.1in
\end{figure}


\section{B. Experiment Details}
\subsection{B.1. Datasets}
\label{apx:dataset}
We conducted experiments on seven widely used graph datasets: Graph-SST5 (OOD) \cite{dataset:sst,DBLP:conf/iclr/WuWZ0C22}, Graph-SST5 (ID) \cite{dataset:sst}, Graph-Twitter (OOD) \cite{dataset:sst,DBLP:conf/iclr/WuWZ0C22}, Graph-Twitter (ID) \cite{dataset:sst}, Graph-SST2 \cite{dataset:sst}, COLLAB \cite{dataset:collab} and REDDIT-MULTI-5K \cite{DBLP:conf/aaai/0001RFHLRG19}.

(1) Graph-SST5 (OOD) dataset \cite{dataset:sst}. Graph-SST5 is a sentiment graph dataset. We follow \cite{DBLP:conf/iclr/WuWZ0C22} to partition the dataset into training and test sets in order to increase the task's difficulty.

(2) Graph-SST5 (ID) dataset \cite{dataset:sst}. The original Graph-SST5 dataset.

(3) Graph-SST2 dataset \cite{dataset:sst}. Similar to the Graph-SST5 dataset, but formulated as a binary classification task.

(4) Graph-Twitter (OOD) dataset \cite{dataset:sst,DBLP:conf/iclr/WuWZ0C22}. Similar to the Graph-SST5 dataset, but sourced from a different data origin. We follow \cite{DBLP:conf/iclr/WuWZ0C22} to partition the dataset into training and test sets in order to increase the task's difficulty.

(5) Graph-Twitter (ID) dataset \cite{dataset:sst}. The original Graph-Twitter dataset.

(6) COLLAB dataset \cite{dataset:collab}. COLLAB is a scientific collaboration dataset where each graph represents a researcher's ego network. In this network, researchers and their collaborators are nodes, and an edge indicates a collaboration between two researchers.

(7) REDDIT-MULTI-5K dataset \cite{DBLP:conf/aaai/0001RFHLRG19}. REDDIT-MULTI-5K is a relational dataset extracted from Reddit, with each graph representing an online discussion thread. In these graphs, users are depicted as nodes, and an edge signifies a situation where one of the two users responded to a comment made by the other user.

\subsection{B.2. Create Label Noise}
\label{apx:label}
Two distinct forms of label noise were strategically incorporated into our study: random PLL label and competitive PLL label. In the context of random PLL label, all erroneous labels are added to the candidate label set with the same probability. Concretely, for the binary classification Graph-SST2 dataset, an experimental framework was simulated, simulating three distinct data annotators each endowed with labeling accuracies of 100\%, 70\%, and 50\%, respectively, reflective of their individual competencies in correctly annotating data instances. This simulation is tailored to mimic real-world complexities.

For PLL datasets with competitive PLL label, a more intricate competitive label noise model was employed. This noise formulation emphasizes the augmentation of selecting labels that are semantically proximate to the ground-truth label, inherently capturing the nuances and intricacies prevalent in data. Illustratively, within the Graph-SST5 dataset, encompassing five labels encompassing various emotional tones: ``very negative'', ``negative'', ``neutral'', ``positive'', and ``very positive'', a deliberate emphasis was placed on elevating the likelihood of selecting labels such as ``neutral'' and ``very positive'' when the underlying ground-truth label pertained to ``positive'', thus emphasizing the thematic proximity between these concepts.

Correspondingly, the Graph-Twitter dataset similarly exhibits a tripartite labeling schema consisting of ``negative'', ``neutral'', and ``positive'' sentiments, where the contiguous labels inherently possess a heightened semantic affinity. Comparable considerations extend to the COLLAB dataset, wherein the labels ``High Energy Physics'', ``Condensed Matter Physics'', and ``Astro Physics'' are featured, with the initial two labels evincing more pronounced semantic coherence.

To introduce competitive noise, a meticulous strategy was devised: labels semantically closest to the ground-truth label are appended with a probability of $\rho$, thereby manifesting a salient noise component, the remaining labels are added to the candidate label set with a probability of $(1-\rho)$, Where $\rho$ represents the level of label ambiguity. Specifically, we constructed datasets with two levels of competitive label noise: $\rho=0.9$ and $\rho=0.7$.

This deliberate orchestration of competitive noise substantiates a dynamic framework that engenders heightened amalgamation of noise sources. The resultant intricacy within the noise profile enhances the evaluative capacity of our model, effectively probing its resilience in discerning discriminative patterns amidst intricate data landscapes.

\subsection{B.3. Setting Details}
\label{apx:setting}
We compared the proposed GPCD with the following methods: 1)ML-PLL \cite{mlpll} constructs a transformation matrix to model the relationship between the candidate label set and the ground-truth label. It utilizes a mutual learning paradigm to coordinate and guide the learning of two classifiers; 2) PiCO \cite{pico} adopts contrastive learning to enhance feature learning and employs a strategy based on class prototypes to update the confidence of candidate labels; 3) CIGA \cite{ciga} proposes an information-theoretic objective based on causality to maximize the extraction of invariant intra-class information from the subgraphs; 4) DISC \cite{DBLP:conf/nips/Fan0MST22} proposes a disentangled GNN framework to capture the causal substructure in severe bias data; 5)ARMA \cite{arma} uses the most basic graph neural network and assigns equal weights to all labels.

We split datasets into 80\% training, 10\% validation, and 10\% test sets. The test set accuracy, linked to the best-performing validation set, determined the final result. The pre-training epochs for all datasets were set to 5. The auxiliary training epochs for the dataset REDDIT-MULTI-5K were set to 50, while the rest of the datasets were set to 20. We set $\xi$ as 5. The parameter $\lambda$ was selected from \{2,3,4,5,6,7\} based on the final results.The loss weights for $\mathcal{L}_{c}$, $\mathcal{L}_{g}$, and $\mathcal{L}_{ce}$ were set to 1, 0.5, and 1, respectively.
During the extraction of potential causes, clustering was stopped if the average mean squared error was less than 0.28.
The initial value of hyperparameter $\Delta$ was set to 0.15, and it was increased after each extraction of latent causes to prevent the explosion of latent variables. The experiments indicate that an initially large  $\Delta$ will result in the failure to extract latent causes. If  $\Delta$ remains constant subsequently, latent causes will experience explosive growth, leading to deteriorated results.
We set t to 4, potential causes were re-extracted every 4 epochs during the auxiliary training process.
The model is trained by a standard Adam optimizer with a learning rate of 0.0001 and a batch size of 32.

All our experiments were conducted on a workstation with two Quadro RTX 5000 GPU (16 GB), one Intel Xeon E5-1650 CPU, 128GB RAM, and a Unbuntu 20.04 operating system. For all experimental results, we conducted five independent runs and reported the mean ± standard deviation.

\subsection{B.4. Implementation of $\mathcal{K}(\cdot)$}
\label{apx:ke}
The function $\mathcal{K}(\cdot)$ employs the K-Means clustering technique on all node representations within ${\psi(G_{i})}_{i=1}^{N}$, resulting in the creation of multiple clusters. Each cluster is representative of a distinct set of variables. Given the variance in node count across diverse graphs, graphs with fewer nodes naturally exhibit a diminished cumulative sum of squared errors. To establish a uniform benchmark for comparison, we employ the mean squared error per node, calculated by dividing the cumulative sum of squared errors by the total node count across graphs. In our specific implementation, we introduce a hyperparameter $\alpha$ denoting the threshold for the Average Mean Squared Error (AMSE). Referring to the design of the elbow method, the clustering process concludes when the AMSE descends below this stipulated threshold. Algorithm \ref{alg:ke} describes the detailed procedure.

\begin{algorithm}
\caption{Pipeline of $\mathcal{K}(\cdot)$.}
\label{alg:ke}
\begin{algorithmic}[1] 
\State All node representations within $\{\psi(G_{i})\}_{i=1}^{N}$.
\State Cluster centers set $C$.
\State Set AMSE $q^{\text{AMSE}}$ as 0, cluster num $k$ as 1.
\For{$q^{\text{AMSE}} > \alpha$}
    \State Apply K-means clustering to partition all node representations within $\{\psi(G_{i})\}_{i=1}^{N}$ into $k$ clusters.
    \State Compute $q^{\text{AMSE}}$ for these $k$ clusters.
    \State $k = k + 10$.
\EndFor
\State Acquired the cluster centers set $C$.
\State \Return $C$.
\end{algorithmic}
\end{algorithm}

\subsection{B.5. Additional Experimental Results}
In Figure \ref{fig:sentiment2} and \ref{fig:sentiment3}, we present supplementary empirical findings concerning both the training and testing sets of the Graph-Twitter dataset. Whether applied to the training or testing partition, GPCD demonstrates a strong knack for identifying words that closely align with the true meaning of sentences. This natural ability allows the model to excel in accurately classifying various text inputs. The practical results presented here highlight GPCD's skill in grasping the real essence of sentences in partial-label datasets.

\subsection{B.6. Computational complexity analysis}
\label{apx:cpx}
The computational complexity of our algorithm needs to be analyzed separately for each of the three phases it encompasses. For phase one, since the only difference from conventional GNNs involves the use of an MLP to map node representations at the feature output stage, its computational complexity is similar to that of other graph neural networks, i.e. $\mathcal{O}( N \times |\mathcal{V}| \times \mathcal{D} \times \mathcal{J})$. Here, $|\mathcal{V}|$ and $\mathcal{D}$ denote the number of averaged nodes and degrees numder in the graph, respectively, while $N$ denotes the number of samples,  $\mathcal{J}$ represents the number of layers. 

In phase two, we implemented clustering operations based on K-means and the elbow method. The computational complexity of the K-means clustering algorithm can be expressed as $\mathcal{O}(\sum_{i=1}^{N}|\mathcal{V}_{i}| \times \mathcal{C} \times  \mathcal{B} \times \bar{\mathcal{B}})$, where $\mathcal{C}$ is the number of clusters, $\mathcal{B}$ is the number of iterations required for convergence of K-means, and $\bar{\mathcal{B}}$ denotes the total number of the K-means conducted. $\mathcal{O}(\sum_{i=1}^{N}|\mathcal{V}_{i}| \times  \mathcal{C} \times \mathcal{I} \times \mathcal{B} \times \bar{\mathcal{B}})$ is equal to $\mathcal{O}(  N \times |\mathcal{V}| \times \mathcal{C} \times \mathcal{B} \times \bar{\mathcal{B}})$. The product $\mathcal{C} \times \mathcal{B} \times \bar{\mathcal{B}})$  is kept within a relatively small range by controlling the number of centers and the frequency of computing k-means. Additionally, this algorithm does not require backpropagation and performs the computation only once per $\lfloor \frac{\mu}{2} \rfloor$ round, therefore the compuation cost is no larger than conduct one epoch of GNN training.

In phase three, similarly, we have only modified the mapping of node representations, therefore the computational complexity remains $\mathcal{O}( N \times |\mathcal{V}| \times \mathcal{D} \times \mathcal{J})$.

In summary, the computational complexity of our method is essentially the same as that of the GNN model.

\paragraph{Model Structure Evaluation.}
We performed an analysis of some parameters employed in our method as a verification of the model structure, as illustrated in Figure \ref{fig:parameter-a}. The results indicate that an appropriate number $\mu$ of pre-training epochs plays a crucial role in facilitating the extraction of meaningful potential causes, ultimately contributing to optimal model performance. Figure \ref{fig:parameter-b} further elucidates that a relatively larger value of the hyperparameter $\lambda$ corresponds to more favorable outcomes, aligning with empirically validated research findings.

\clearpage

\section{Reproducibility Checklist}

This paper:

\begin{itemize}
\item Includes a conceptual outline and/or pseudocode description of AI methods introduced. (yes)

\item Clearly delineates statements that are opinions, hypothesis, and speculation from objective facts and results. (yes)

\item Provides well marked pedagogical references for less-familiare readers to gain background necessary to replicate the paper. (yes)

\item Does this paper make theoretical contributions? (yes)

\item All assumptions and restrictions are stated clearly and formally. (yes)

\item All novel claims are stated formally (e.g., in theorem statements). (yes)

\item Proofs of all novel claims are included. (yes)

\item Proof sketches or intuitions are given for complex and/or novel results. (yes)

\item Appropriate citations to theoretical tools used are given. (yes)

\item All theoretical claims are demonstrated empirically to hold. (yes)

\item All experimental code used to eliminate or disprove claims is included. (yes)

\item Appropriate citations to theoretical tools used are given. (yes) 

\item Does this paper rely on one or more datasets? (yes)

\item A motivation is given for why the experiments are conducted on the selected datasets. (yes)

\item All novel datasets introduced in this paper are included in a data appendix. (yes)

\item All novel datasets introduced in this paper will be made publicly available upon publication of the paper with a license that allows free usage for research purposes. (yes)

\item All datasets drawn from the existing literature (potentially including authors’ own previously published work) are accompanied by appropriate citations. (yes)

\item All datasets drawn from the existing literature (potentially including authors’ own previously published work) are publicly available. (yes)

\item All datasets that are not publicly available are described in detail, with explanation why publicly available alternatives are not scientifically satisficing. (yes)

\item Does this paper include computational experiments? (yes)

\item Any code required for pre-processing data is included in the appendix. (yes).

\item All source code required for conducting and analyzing the experiments is included in a code appendix. (yes)

\item All source code required for conducting and analyzing the experiments will be made publicly available upon publication of the paper with a license that allows free usage for research purposes. (yes)

\item All source code implementing new methods have comments detailing the implementation, with references to the paper where each step comes from (yes)

\item If an algorithm depends on randomness, then the method used for setting seeds is described in a way sufficient to allow replication of results. (yes)

\item This paper specifies the computing infrastructure used for running experiments (hardware and software), including GPU/CPU models; amount of memory; operating system; names and versions of relevant software libraries and frameworks. (yes)

\item This paper formally describes evaluation metrics used and explains the motivation for choosing these metrics. (yes)

\item This paper states the number of algorithm runs used to compute each reported result. (yes)

\item Analysis of experiments goes beyond single-dimensional summaries of performance (e.g., average; median) to include measures of variation, confidence, or other distributional information. (yes)

\item The significance of any improvement or decrease in performance is judged using appropriate statistical tests (e.g., Wilcoxon signed-rank). (yes)

This paper lists all final (hyper-)parameters used for each model/algorithm in the paper’s experiments. (yes)

\item This paper states the number and range of values tried per (hyper-) parameter during development of the paper, along with the criterion used for selecting the final parameter setting. (yes)
\end{itemize}

\end{document}